%% file: main.tex
\documentclass[11pt]{article}
\usepackage[margin=1in]{geometry}
\usepackage{times}
\usepackage[authoryear,round]{natbib}
\usepackage{authblk}
\usepackage{hyperref}
\usepackage{url}
\usepackage{multirow}
\usepackage{graphicx}
\usepackage{algorithm}
\usepackage{algorithmic}
\usepackage{amsmath}
\usepackage{amsthm}
\usepackage{amssymb}
\usepackage{bm}
\usepackage{xcolor}
\usepackage{enumitem}
\input{math_commands.tex}

\newtheorem{definition}{Definition}
\newtheorem{theorem}{Theorem}
\newtheorem{corollary}{Corollary}

\allowdisplaybreaks

\title{Protecting Sensitive Data in Image Synthesis via PAC-Private Adaptation for Diffusion Models}
\author[1]{Boming Miao}
\author[1]{Tao Zhang}
\author[1]{Netanel Raviv}
\author[2]{Murat Kantarcioglu}
\author[3]{Bradley A. Malin}
\author[1]{Yevgeniy Vorobeychik}
\affil[1]{Department of Computer Science and Engineering, Washington University in St. Louis\\\texttt{\{bmiao,taoz,netanel.raviv,yvorobeychik\}@wustl.edu}}
\affil[2]{Department of Computer Science, Virginia Tech\\\texttt{muratk@vt.edu}}
\affil[3]{Department of Biostatistics, Vanderbilt University\\\texttt{b.malin@vumc.org}}
\date{}

\begin{document}

\maketitle
\begin{abstract}
Synthetic data are increasingly used as an alternative to sharing sensitive records. However, synthetic data generation does not guarantee privacy, as diffusion models trained or adapted on sensitive data remain susceptible to reconstruction attacks. Moreover, while approaches that use differential privacy (DP), such as DP-SGD, achieve provably private diffusion model training, the repeated gradient clipping and noise injection they require result in significant utility loss. An important limitation of DP-based privacy is that, although it has a provable relationship to reconstruction privacy (RP), that relationship is indirect. RP is defined in terms of limiting how much an adversary’s posterior distribution over sensitive data  differs from the prior, whereas DP provides guarantees by bounding the sensitivity of outputs to changes in individual records. This indirection is an important source of the utility loss. To address this, we propose a PAC-private diffusion model adaptation to achieve reconstruction privacy. Since PAC-privacy is defined directly with respect to posterior advantage over the prior, it directly implicates RP. To obtain scalable PAC privatization in high dimensions, we first learn a compact data-dependent diffusion model component using LoRA or Textual Inversion, and then calibrate anisotropic Gaussian noise from the covariance of repeated mechanism outputs. Unlike DP-SGD, our method perturbs the learned component only once after optimization, thereby avoiding privacy composition across gradient updates. We evaluate the framework on few-shot concept personalization and full-dataset image synthesis, and show that the proposed approach better preserves subject identity, generation quality, and downstream classification accuracy than DP while achieving the same reconstruction privacy. 
\end{abstract}

\input{sections/1.intro}

\input{sections/2.relwork}

\input{sections/3.prelim}

\input{sections/4.approach}

\input{sections/5.experiment}
\input{sections/6.conclusion}

\bibliography{references}
\bibliographystyle{plainnat}

\appendix
\input{sections/appendixa}

\input{sections/appendixb}

\input{sections/appendixc}
\input{sections/appendixd}

\end{document}

%% file: math_commands.tex
\usepackage{amsmath,amsfonts,bm}

\def\eqref#1{equation~\ref{#1}}
\def\1{\bm{1}}

\DeclareMathAlphabet{\mathsfit}{\encodingdefault}{\sfdefault}{m}{sl}
\SetMathAlphabet{\mathsfit}{bold}{\encodingdefault}{\sfdefault}{bx}{n}

%% file: sections/1.intro.tex
\section{Introduction}

Synthetic data provide an appealing alternative to the direct use of raw sensitive data in downstream applications, such as clinical data analytics, allowing useful data to be generated and shared without directly exposing the original records~\citep{hu2024sok}. For image synthesis, diffusion models have emerged as a leading generative paradigm for synthetic data, demonstrating strong performance and supporting a wide range of modern image-generation applications~\citep{nichol2021glide, rombach2022high, podell2024sdxl}.

However, 
diffusion-based data synthesis does not necessarily prevent the reconstruction of sensitive training data~\citep{carlini2023extracting,stadler2022synthetic,webster2023reproducible}.
Indeed, diffusion models may even memorize individual training examples~\citep{matsumoto2023membership}.
Accordingly, for synthetic data to serve as a privacy-preserving alternative to the direct use of raw data, the synthesis process itself must provide an explicit privacy guarantee with respect to the sensitive training data. 
A widely used framework for providing such formal privacy guarantees for diffusion-based synthesis is differential privacy (DP)~\citep{dwork2006our}, with DP-SGD~\citep{abadi2016deep,li2024privimage,liu2024efficient,liu2024differentially} commonly employed to train or fine-tune diffusion models on sensitive data. 
DP-SGD bounds per-example gradient contributions through clipping and adds Gaussian noise during private optimization. 
However, differentially private training also introduces significant utility degradation.
Several approaches aim to reduce this degradation, for example, DP-Promise~\citep{wang2024dp}, which exploits the noise inherent in the diffusion forward process, and DP-FETA~\citep{li2025easy}, which adds a private preprocessing step to warm start diffusion model training. Given the widespread use of latent diffusion models for scalable, high-resolution image synthesis~\citep{rombach2022high}, we focus on privacy-preserving synthesis in this practically important setting.
However, these approaches work by exploiting the properties of images, and cannot be directly extended to latent diffusion, which requires additional sensitivity control because the VAE latent representations need not have a bounded input domain.

Stepping back, there are two principal sources of utility loss in DP-based diffusion model training.
First, privacy must be enforced at every optimization step, using up privacy budget.
Second, DP is a worst-case and somewhat abstract notion of privacy which considers output sensitivity to any single-record changes.
However, the main motivation for generating synthetic data is to avoid revealing real sensitive data, which naturally translates to the goal of preventing successful reconstruction of original data or sensitive attributes within it, or \emph{reconstruction privacy (RP)}.
RP, in turn, can be formally captured as ensuring that the attacker's posterior distribution over the training data conditional on observing the diffusion model, or the synthetic data generated by it, does not significantly differ from the prior data distribution---as measured, for example, using KL-divergence.
While bounds on this difference can be derived from DP guarantees (as we show), the indirectness which results (since no bounds are entirely tight, and worst-case adversarial assumptions of DP are stronger than what we need for RP) entails significant utility loss.

Together, these considerations motivate a different route to privacy-preserving latent diffusion synthesis.
First, rather than enforcing privacy at every optimization step, 
we calibrate it directly to the learned adaptation that is ultimately used for privacy-preserving synthetic image generation.
Because the synthetic images are generated from the final learned adaptation, the data-processing inequality allows us to enforce this objective by privatizing the adaptation itself. Specifically, instead of clipping and perturbing gradients throughout optimization, we apply a single calibrated perturbation to the learned adaptation parameters after training, thereby avoiding explicit privacy composition across gradient updates. 
Second, since applying standard DP through output perturbation requires calibration to the worst-case sensitivity of the learned parameters, it tends to be overly conservative and result in significant utility loss.
To address this, we instead use an information-theoretic privacy notion, PAC privacy~\citep{xiao2023pac}, as the building block.
PAC privacy is defined directly in terms of the adversarial posterior reconstruction advantage over the prior, thereby providing a natural formalization of RP.
Additionally, we enable a principled comparison with DP by showing that $(\varepsilon,\delta)$-DP bounds the reconstruction advantage.

The key remaining technical challenge in applying the privatization procedure to achieve PAC privacy is that it scales poorly with dimension as a result of a Monte Carlo calibration step it requires~\citep{xiao2023pac,sridhar2025pac,zhang2026residual}.
To make it practical in the high-dimensional image synthesis setting, 
we employ parameter-efficient diffusion adaptation. For tasks that require parameter fine-tuning, we learn compact LoRA adapters instead of optimizing the full parameter space. Because functionally equivalent LoRA factors can exhibit different scalings, we canonicalize the factors before covariance estimation to remove this ambiguity and improve calibration stability. For simpler adaptation tasks that can be represented through text conditioning, we instead privatize the special-token embedding learned by Textual Inversion~\citep{gal2022image}. In both cases, PAC calibration is performed on a substantially lower-dimensional learned component rather than on the full diffusion model.

We evaluate the proposed framework in two complementary settings. For few-shot personalized generation, we compare PAC output perturbation with DP-SGD using both Textual Inversion and LoRA, and evaluate concept or identity fidelity, text alignment, and perceptual image quality. For general-purpose synthetic data generation, we fine-tune diffusion models on private datasets using LoRA and evaluate the downstream classification accuracy of models trained on generated data. Across these settings, our results show that the proposed PAC-based final-output privatization typically preserves substantially greater utility than DP-SGD under matched privacy budgets.

In summary, we make the following contributions:

\begin{itemize}[topsep=0pt,itemsep=0pt,leftmargin=*]
    \item We develop end-to-end privacy-preserving image synthesis based on PAC-private diffusion-model adaptation that avoids per-example perturbation over individual gradient updates.
    \item  We develop a scalable PAC calibration framework for parameter-efficient diffusion adaptation using LoRA and Textual Inversion. For LoRA, we canonicalize the learned factors before calibration to remove scale and sign ambiguities, yielding more stable covariance estimation and improved downstream utility.
    \item Experimental results demonstrate that our PAC-private diffusion-adaptation framework preserves substantially greater utility than DP-SGD in challenging high-resolution and small-data regimes.
\end{itemize}

%% file: sections/2.relwork.tex
\section{Related Work}

\noindent\textbf{Differentially Private Diffusion Models.}
Private diffusion training is commonly implemented using DP-SGD. DPDM~\citep{dockhorn2022differentially} introduces noise multiplicity, which averages gradients over multiple diffusion noise realizations of each training example without additional privacy cost. \citet{ghalebikesabi2023differentially} improve private diffusion fine-tuning through public pretraining, augmentation multiplicity, and biased timestep sampling. Other approaches exploit the structure of diffusion training more directly, e.g., DP-Promise~\citep{wang2024dp}, which incorporates the forward diffusion process into its privacy analysis, and DP-FETA~\citep{li2025easy}, which adopts an easy-to-hard curriculum that first learns coarse structure from privately aggregated central images and subsequently performs DP-SGD training. As image resolution and model scale increase, private diffusion training becomes expensive. Another line of work therefore improves efficiency by leveraging pretrained diffusion models and restricting private optimization to a small subset of parameters. \citet{jiang2025rapid} use public data to construct a library of denoising trajectories, retrieve a similar public trajectory for each private training example, and apply DP training only to the remaining refinement stage. DP-LDM~\citep{liu2024differentially} performs private training in latent space and fine-tunes only the attention modules. \citet{liu2024efficient} propose a low-dimensional adaptation strategy for computationally efficient private fine-tuning, and DP-LoRA~\citep{tsai2025differentially} applies DP-SGD to LoRA modules in pretrained latent diffusion models and demonstrates that parameter-efficient private fine-tuning can achieve a better privacy-utility trade-off than full-parameter fine-tuning.

\noindent\textbf{PAC Privacy and Its Applications.}
Our approach uses tools from PAC Privacy~\citep{xiao2023pac}, which provides a framework that limits the improvement in an adversary's ability to recover private data after observing the output of a mechanism. It also enables automatic privacy analysis of arbitrary black-box data-processing mechanisms. \citet{sridhar2025pac} further develop a practical framework for automatically privatizing such mechanisms through output canonicalization and calibrated perturbation. Subsequent work extends PAC Privacy to database systems~\citep{sridhar2026pac,battiston2026simd}, while \citet{zhu2026pac} establishes linear mutual information composition for adaptively and adversarially selected PAC-private model responses. \citet{zhang2026residual} show that the Gaussian log-determinant bound used by existing PAC mechanisms is in general conservative for non-Gaussian outputs and introduce Residual-PAC Privacy with an automatic privatization algorithm termed Stackelberg PAC privatization to tightly implement the target privacy budget with less noise.
PAC Privacy guarantees have also been studied for learnable data encodings~\citep{xiao2024formal}, constrained one-sided perturbations~\citep{xiao2025one}, verifiable privatization using zero-knowledge proofs~\citep{repetto2026pac}, and autoregressive generation~\citep{mirzadehsarcheshmeh2026pac}. More recently, \citet{ertan2026paczero} apply PAC Privacy to sign-quantized language-model fine-tuning. In diffusion-based generation, P3DM~\cite{xu2024pac} studies PAC privacy through data perturbation and calibrates Gaussian noise based on the covariance of image space. However, directly perturbing data in pixel space can substantially degrade the image quality, thereby reducing the utility of the resulting model for downstream tasks. In contrast, we perturb only the compact, privacy-dependent adaptation parameters while leaving the training images unchanged.

%% file: sections/3.prelim.tex
\section{Preliminaries}
\label{S:prelim}

A major privacy concern in diffusion model adaptation for generating synthetic data is that someone (an ``adversary'') with access to the adapted diffusion model or the data it generates should not be able to reconstruct the original sensitive training data \emph{beyond what an adversary armed solely with the prior could}.
For example, if the training data includes sensitive faces as well as common non-sensitive background elements, such as trees and cars, we are only concerned that the faces should not be effectively reconstructed.
In this section, we formalize this notion of \emph{reconstruction privacy (RP)} in information-theoretic terms; it then serves as our underlying privacy goal for diffusion model adaptation.
Subsequently, we overview the DP tools for achieving RP, as well as diffusion models.

\noindent\textbf{Reconstruction Privacy.}
We begin by formalizing reconstruction privacy.
The key idea RP captures is that an adversary with access to a model trained on sensitive data should learn little more about these data than if they only knew the prior data distribution.
Let $\mathcal X$ denote the sample space and $\mathcal{D}$ the distribution over $\mathcal X$.
Define $X=(x_1,\ldots,x_n)$ as the private dataset with each $x_i \in \mathcal X$ drawn i.i.d.~from $\mathcal{D}$.
We capture reconstruction attack success in terms of the attacker successfully identifying an estimate $\hat{x}$ which is close to some data point in $X$.
To formalize, let $d(\hat{x},x)$ be the distance (e.g., $\ell_p$ norm) between a pair $\widehat{x},x \in \mathcal X$ and $r$ a reconstruction threshold, such that the estimate $\hat{x}$ is successful if $\min_{x \in X} d(\widehat{x},x) \le r$.

Let $M(X)$ be a function that maps a dataset $X$ to some output space $\mathcal O$.
In our setting, $M$ is the algorithm for diffusion model adaptation outputting an adapter. The privacy goal is that an attacker who observes the output of $M$ should not thereby obtain much more \emph{information} about $X$ than through knowing $\mathcal{D}$ alone. More precisely, let $\delta(M)$ be the attacker-optimal posterior reconstruction error defined as $\delta(M)=1-\sup_{\mathcal A}\Pr_{X\sim\mathcal D^n}[\min_{x \in X} d(\mathcal A(M(X)),x)\le r]$, where $\mathcal{A}$ is a reconstruction function, whereas the optimal prior reconstruction error is $\delta^{o}=1-\sup_{\widehat{x}\in\mathcal X}\Pr_{X\sim\mathcal D^n}[\min_{x \in X}d(\widehat{x},x)\le r]$.
Treating prior and posterior reconstruction errors in distributional terms, we wish that the difference between these two distributions is small, naturally formalized in terms of KL-divergence as in PAC privacy~\citep{xiao2023pac} of which our notion of $\beta$-reconstruction privacy is a special case. 
\begin{definition}[$\beta$-reconstruction privacy]
Let $\Delta_{\mathrm{KL}}(\delta^o,\delta(M))\equiv D_{\mathrm{KL}}(\operatorname{Bern}(1-\delta(M))\|\operatorname{Bern}(1-\delta^o))$ be the \emph{reconstruction advantage} of $M$, where $D_{\mathrm{KL}}$ is the KL-divergence and $\operatorname{Bern}(\cdot)$ the Bernoulli distribution.
Then $M$ satisfies $\beta$-reconstruction privacy ($\beta$-RP) if $\Delta_{\mathrm{KL}}(\delta^o,\delta(M)) \le \beta$.
\end{definition}
Moreover, $\Delta_{\mathrm{KL}}(\delta^o,\delta(M))\leq I(X;M(X))$~\citep{xiao2023pac}, where $I(X;M(X))$ is the mutual information between $X$ and $M(X)$~\citep{shannon1948mathematical}.
Thus, it suffices to bound $I(X;M(X))$
to achieve $\beta$-RP.

\noindent\textbf{Differential Privacy.}
Unlike RP, differential privacy (DP) enforces indistinguishability between the output distributions induced by neighboring datasets that differ in one sample.
\begin{definition}[$(\varepsilon,\delta)$-differential privacy~\citep{dwork2006our}]
A randomized mechanism $M$ satisfies $(\varepsilon,\delta)$-differential privacy if for any neighboring datasets $X\simeq X'$ differing in the addition or removal of one sample, and any measurable set $S\subseteq\mathcal O$, $\Pr[M(X)\in S]\leq e^\varepsilon\Pr[M(X')\in S]+\delta$.
\end{definition}

For private neural network training, DP-SGD~\citep{abadi2016deep} is the most widely adopted approach that applies the Gaussian mechanism to clipped stochastic gradients.
In learning step $t$, let $g_t^{(i)}=\nabla_{\theta_t}\ell(\theta_t;x_i)$ denote the gradient of sample $x_i$. Each gradient is clipped as $\operatorname{clip}_C(g_t^{(i)})=g_t^{(i)}\min\{1,C/\|g_t^{(i)}\|_2\}$, and for a minibatch $\mathcal B_t$, DP-SGD computes $\widetilde g_t=|\mathcal B_t|^{-1}\{\sum_{i\in\mathcal B_t}\operatorname{clip}_C(g_t^{(i)})+Z_t\}$, where $Z_t\sim\mathcal N(0,\sigma_{\mathrm{DP}}^2C^2I)$ and $\sigma_{\mathrm{DP}}$ is obtained from privacy accounting, and updates the model using $\widetilde g_t$.

\noindent\textbf{Diffusion Models.}
Diffusion models learn to generate data by reversing a progressive Gaussian noising process~\citep{ho2020denoising, song2020denoising}, often implemented in a lower-dimensional latent space to improve efficiency~\citep{rombach2022high}. 
Specifically, given an image $x$, a pretrained encoder $\mathcal E$ produces its latent representation $z_0=\mathcal E(x)$. For a variance schedule $\{\beta_t\}_{t=1}^{T}$, let $\alpha_t=1-\beta_t$ and $\bar{\alpha}_t=\prod_{s=1}^{t}\alpha_s$. A noisy latent at timestep $t$ is sampled as $z_t=\sqrt{\bar{\alpha}_t}z_0+\sqrt{1-\bar{\alpha}_t}\epsilon$, where $\epsilon\sim\mathcal N(0,I)$. For text-conditioned generation, a text encoder maps a prompt $y$ to a conditioning representation $c(y)$. The denoising network $\epsilon_\theta$ is trained to predict the injected noise by minimizing $\mathcal L(\theta)=\mathbb E_{x\sim\mathcal D,\,t\sim\mathrm{Unif}([T]),\,\epsilon\sim\mathcal N(0,I)}[\|\epsilon-\epsilon_\theta(z_t,t,c(y))\|_2^2]$.
During generation, the reverse process starts from $z_T\sim\mathcal N(0,I)$. 
We use classifier-free guidance~\citep{ho2022classifier}, which combines conditional and unconditional predictions as $\widehat{\epsilon}_\theta=\epsilon_\theta(z_t,t,c(\varnothing))+w[\epsilon_\theta(z_t,t,c(y))-\epsilon_\theta(z_t,t,c(\varnothing))]$, where $w$ is the guidance scale. A diffusion sampler iteratively estimates the noisy latent at the previous timestep, and the final latent is decoded as $\tilde{x}=\mathcal G(z_0)$, where $\mathcal G$ is the decoder jointly trained with $\mathcal{E}$.

%% file: sections/4.approach.tex
\section{Reconstruction-Private Diffusion Adaptation}
\label{sec:pac_ti}

\subsection{Overall Framework}

Training a diffusion model from scratch on private data requires many optimization steps to learn basic denoising behavior and general visual structure, a process difficult to make private without considerable sacrifice in utility. 
Instead, we initialize the model from publicly pretrained parameters $\theta_0$ and use the private dataset only to learn a compact adaptation component $\phi$. 
Our approach for training $\phi$ proceeds in two phases: the noise calibration phase, in which we estimate the covariance matrix of an adaptation algorithm for $\phi$, and the privatization phase, in which the trained adaptation is privatized by adding Gaussian noise to obtain a provable $\beta$-RP guarantee.

\noindent\textbf{Noise calibration.}
Our first step involves estimating the covariance of an adapter $\phi$ trained on bootstrapped subsamples of the dataset $X$.
Specifically, given the public initialization $\theta_0$, let $\mathrm{Train}$ denote the standard gradient-based training procedure for $\phi$, and let $\mathcal Q$ denote the bootstrap distribution induced by the private training set. 
We repeatedly sample $X^{(k)}\sim\mathcal Q(X_{\mathrm{train}})$ and run the training algorithm 
to obtain $\phi^{(k)}=\mathrm{Train}(\theta_0,X^{(k)})$. Let $\bar\phi=\sum_{k=1}^{K}\phi^{(k)}/K$. We estimate the covariance of the trained adapter $\phi$ (viewed as a random vector) as $\widehat\Sigma=\sum_{k=1}^{K}(\phi^{(k)}-\bar\phi)(\phi^{(k)}-\bar\phi)^\top/(K-1)$. We progressively increase $K$ until the relative change in $\operatorname{tr}(\widehat\Sigma)$ falls below the convergence threshold $\tau$, and then perform the eigendecomposition $\widehat\Sigma=V\operatorname{diag}(\lambda_1,\ldots,\lambda_d)V^\top$. 
The noise variance along direction $V_i$ is set to $e_i=\sqrt{\lambda_i}\sum_{j=1}^{d}\sqrt{\lambda_j}/(2\beta)$, giving $\widehat\Sigma_B=V\operatorname{diag}(e_1,\ldots,e_d)V^\top$. Algorithm~\ref{alg:diffusion_pac_calibration} summarizes the complete procedure.

\begin{algorithm}[ht]
\caption{Noise Calibration for Diffusion Adapters}
\label{alg:diffusion_pac_calibration}
\begin{algorithmic}[1]
\REQUIRE Private dataset $X_{\mathrm{train}}$; frozen public model $\theta_0$ and pretrained encoder $\mathcal E$; trainable component $\phi\in\mathbb R^d$ with fixed initialization $\phi_{\mathrm{init}}$; adaptation steps $S$; learning rate $\eta$; bootstrap distribution $\mathcal Q(X_{\mathrm{train}})$ that samples $|X_{\mathrm{train}}|$ records with replacement; mutual-information budget $\beta$; tolerance $\tau$
\ENSURE Noise covariance $\widehat\Sigma_B$

\STATE $\Phi\gets\varnothing$, $K\gets0$, $c\gets\infty$, $\nu_{\mathrm{old}}\gets\bot$

\WHILE{$c\geq\tau\ \lor\ K\leq2$}
    \STATE $K\gets K+1$
    \STATE $X^{(K)}\sim\mathcal Q(X_{\mathrm{train}})$; $\phi_0\gets\phi_{\mathrm{init}}$

    \FOR{$s=0,\ldots,S-1$}
        \STATE Sample a minibatch
        $\mathcal B_s=\{(x_{s,j},y_{s,j})\}_{j=1}^{B}$ from $X^{(K)}$
        \STATE Independently sample $t_{s,j}\sim\operatorname{Unif}([T])$ and
        $\epsilon_{s,j}\sim\mathcal N(0,I)$ for each $j=1,\ldots,B$
        \STATE $z_{s,j}^{0}\gets\mathcal E(x_{s,j})$;
        $z_{s,j}^{t}\gets\sqrt{\bar\alpha_{t_{s,j}}}z_{s,j}^{0}
        +\sqrt{1-\bar\alpha_{t_{s,j}}}\epsilon_{s,j}$
        \STATE $\mathcal L_s\gets \sum_{j=1}^{B}
        \lVert\epsilon_{s,j}
        -\epsilon_{\theta_0,\phi_s}
        (z_{s,j}^{t},t_{s,j},c(y_{s,j}))\rVert_2^2/B$
        \STATE $\phi_{s+1}\gets\phi_s-\eta\nabla_{\phi_s}\mathcal L_s$
        \hfill\textit{\# Update only the adaptation component}
    \ENDFOR

    \STATE $\phi^{(K)}\gets\operatorname{vec}(\phi_S)$; $\Phi\gets\Phi\cup\{\phi^{(K)}\}$

    \IF{$K\geq2$}
        \STATE $\bar\phi\gets \sum_{k=1}^{K}\phi^{(k)}/K$;
        $\widehat{\Sigma}\gets\sum_{k=1}^{K}
        (\phi^{(k)}-\bar\phi)(\phi^{(k)}-\bar\phi)^\top/(K-1)$; $\nu\gets\operatorname{tr}(\widehat\Sigma)$
        \STATE $c\gets\infty$ if $\nu_{\mathrm{old}}=\bot$; otherwise $c\gets|\nu-\nu_{\mathrm{old}}|/\max\{\nu_{\mathrm{old}},10^{-12}\}$
        \STATE $\nu_{\mathrm{old}}\gets\nu$
    \ENDIF
\ENDWHILE

\STATE Apply the eigendecomposition
$\widehat{\Sigma}
=V\operatorname{diag}(\lambda_1,\ldots,\lambda_d)V^\top$
\STATE $e_i\gets\sqrt{\lambda_i}\sum_{j=1}^{d}\sqrt{\lambda_j}/(2\beta)$ for $i=1,\ldots,d$
\STATE $\widehat\Sigma_B\gets V\operatorname{diag}(e_1,\ldots,e_d)V^\top$
\RETURN $\widehat\Sigma_B$
\end{algorithmic}
\end{algorithm}

\noindent\textbf{Private release and generation.}
Once the noise covariance $\widehat{\Sigma}_B$ has been determined, let $X^\star\sim\mathcal D^n$ denote the private dataset used for the final adaptation. We run the training procedure once more to obtain $\phi^\star=\mathrm{Train}(\theta_0,X^\star)$. 
Finally, we release the privatized adaptation $\tilde{\phi} = \phi^\star+B$, where $B\sim\mathcal{N}(0,\widehat \Sigma_B)$.
Next, we show that our full procedure yields $\beta$-RP.
\begin{theorem}[Privacy Guarantee]
\label{thm:pac_diffusion_release}
Let $\Sigma=\operatorname{Cov}(\phi^\star)$ denote the population covariance. Let $\Sigma_B$ be obtained by applying the noise allocation step of Algorithm~\ref{alg:diffusion_pac_calibration} to $\Sigma$ at budget $\beta$. For $B\sim\mathcal N(0,\Sigma_B)$ independent of $X^\star$ and $\phi^\star$, the release $\widetilde\phi=\phi^\star+B$ satisfies $I(X^\star;\widetilde\phi)\leq\beta$. For any collection of images $Y$ generated using $\widetilde\phi$, $I(X^\star;Y)\leq I(X^\star;\widetilde\phi)\leq\beta$.
\end{theorem}
\begin{proof}
The mutual-information argument follows the PAC privacy calibration analysis of \citet{sridhar2025pac}. We have the Markov chain $X^\star\rightarrow\phi^\star\rightarrow\widetilde{\phi}$, and therefore $I(X^\star;\widetilde{\phi})
\leq I(\phi^\star;\widetilde{\phi})$. The Gaussian maximum-entropy bound gives
\[
I(\phi^\star;\widetilde{\phi})
\leq
\frac{1}{2}\log\det
\bigl(I+\Sigma\Sigma_B^{-1}\bigr)
=
\frac{1}{2}\sum_{i=1}^{d}
\log\left(1+\frac{\lambda_i}{e_i}\right)
\leq
\frac{1}{2}\sum_{i=1}^{d}\frac{\lambda_i}{e_i}=\beta.
\]
Finally, generation forms the Markov chain
$X^\star\rightarrow\widetilde{\phi}\rightarrow Y$.
The data-processing inequality therefore gives
$I(X^\star;Y)\leq I(X^\star;\widetilde{\phi})\leq\beta$.
\end{proof}
Finally, $\beta$-RP follows immediately by $\Delta_{\mathrm{KL}}(\delta^o,\delta(M))\leq I(X;M(X))$ where $M$ can be either the full training and privatization procedure for $\tilde{\phi}$ or this procedure composed with image generation.

\subsection{Parameter-Efficient Diffusion Adaptation}
The approach above relies on training an adapter $\phi$.
We instantiate it using two concrete adaptation methods: low-rank adaptation (LoRA)~\citep{hu2021lora} and textual inversion.

\noindent\textbf{Low-Rank Adaptation.}
For tasks requiring greater adaptation capacity, we apply Low-Rank Adaptation (LoRA) only to the key and value projections of the cross-attention modules in the denoising network. For a pretrained weight $W_0\in\mathbb R^{d_{\mathrm{out}}\times d_{\mathrm{in}}}$, LoRA parameterizes the adapted weight as $W=W_0+\frac{\alpha}{r}BA$, where $A\in\mathbb R^{r\times d_{\mathrm{in}}}$, $B\in\mathbb R^{d_{\mathrm{out}}\times r}$, and $r$ is the LoRA rank. We keep the remaining model parameters frozen and optimize only $A$ and $B$ on the private dataset, with $A$ randomly initialized and $B$ initialized to zero.

Directly calibrating the raw factors is unstable because different factorizations can represent the same update, since $BA=(c^{-1}B)(cA)$ for any $c\neq0$. We therefore canonicalize each learned pair before covariance estimation. Given $BA=U\operatorname{diag}(s)V^\top$, we set $B^{\mathrm{can}}=U\operatorname{diag}(\sqrt{s})$ and $A^{\mathrm{can}}=\operatorname{diag}(\sqrt{s})V^\top$, with singular values ordered and signs fixed deterministically. This transformation preserves the effective update while removing functionally irrelevant scaling and sign ambiguities. 
We estimate the covariance and apply privatization in this canonicalized parameter space before restoring the privatized factors to their corresponding attention layers for generation.

\noindent\textbf{Textual Inversion.}
When a private concept can be learned without modifying the diffusion model, we adopt Textual Inversion (TI)~\citep{gal2022image}, which introduces a pseudo-token $S_*$ with a learnable embedding $v$. Given private reference images $X$ and a prompt $y$ containing $S_*$, TI freezes the diffusion model, text encoder, and all other token embeddings and optimizes only $v$ as $v^\star(X)=\arg\min_v\mathbb E_{x\sim X,t,\epsilon}[\|\epsilon-\epsilon_{\theta_0}(z_t,t,c(y;v))\|_2^2]$, where $t\sim\mathrm{Unif}(\{1,\ldots,T\})$ and $\epsilon\sim\mathcal N(0,I)$.

\subsection{Privacy Budget Matching}

We can also leverage $(\varepsilon,\delta)$-DP to achieve $\beta$-RP guarantees.
As shown by \citet{dong2022gaussian}, $(\varepsilon,\delta)$-DP is equivalent to $f_{\varepsilon,\delta}$-DP where $f_{\varepsilon,\delta}(\alpha)=\max\{0,1-\delta-e^\varepsilon\alpha,e^{-\varepsilon}(1-\delta-\alpha)\}$.
If the prior reconstruction success probability is $\alpha$, $f_{\varepsilon,\delta}$-DP bounds the posterior success probability by $1-f_{\varepsilon,\delta}(\alpha)$~\citep{kulynych2025unifying}.
Thus, for any $f_{\varepsilon,\delta}(\alpha) > 0$, setting $\beta(\alpha) = D_{\mathrm{KL}}(\mathrm{Bern}(1-f_{\varepsilon,\delta}(\alpha))||\mathrm{Bern}(\alpha))$ ensures $\beta(\alpha)$-RP, allowing us to directly compare the bounds implied by DP and our approach.
In our experiments, we use $\alpha^\star=(1-\delta)/(1+e^\varepsilon)$, with $1-f_{\varepsilon,\delta}(\alpha^\star)=(e^\varepsilon+\delta)/(1+e^\varepsilon)$ for $(\varepsilon,\delta)$-DP, which gives
$\beta_{\varepsilon,\delta}=D_{\mathrm{KL}}(\operatorname{Bern}((e^\varepsilon+\delta)/(1+e^\varepsilon))\,\|\,\operatorname{Bern}((1-\delta)/(1+e^\varepsilon)))$ as the mutual-information budget to use for our approach.

%% file: sections/5.experiment.tex
\section{Experiments}

\subsection{Experimental Setup}

We evaluate the proposed approach, which we refer to as \textsc{PAC} as it builds on PAC privacy tools, in comparison with $(\varepsilon,\delta)$-DP for diffusion model adaptation under reconstruction risk.
In our evaluation, we fix $\varepsilon$ and $\delta$ for DP, use these to recover $\beta(\varepsilon,\delta)$ for our approach (\textsc{PAC}), thereby leveling the corresponding privacy guarantees, and present our results by comparing utility between DP and PAC across tasks, datasets, and efficacy measures.

\noindent\textbf{Tasks and datasets.}
Our evaluation uses DreamBooth~\citep{ruiz2023dreambooth}, CustomConcept101~\citep{kumari2023multi}, and CelebA-HQ~\citep{liu2015deep,karras2017progressive} for few-shot personalized generation. For general data synthesis, we use FashionMNIST~\citep{xiao2017fashion}, EuroSAT~\citep{helber2019eurosat}, and BloodMNIST from MedMNIST v2~\citep{yang2023medmnist}.

\noindent\textbf{Backbones and adaptation mechanisms.}
We use Stable Diffusion V-1.5~\citep{rombach2022high} for few-shot personalized generation with both LoRA and Textual Inversion, and BK-SDM-Small~\citep{kim2024bk} for LoRA-based general data synthesis. Textual Inversion optimizes only the embedding of a special token. LoRA adapts the key and value projections in the cross-attention layers using rank-$1$ factors ($r=1$), while keeping the pretrained backbone frozen. 
Detailed hyperparameters are provided in the Appendix~\ref{app:experimental_details}.

\noindent\textbf{Evaluation.}
For object personalization on DreamBooth and CustomConcept101, we evaluate subject fidelity using DINO-I~\citep{caron2021emerging,oquab2023dinov2} and CLIP-I~\citep{radford2021learning}, text alignment using CLIP-T~\citep{radford2021learning}, and perceptual quality using BRISQUE~\citep{mittal2012no}. For face personalization, following Anti-DreamBooth~\citep{van2023anti}, we select 50 CelebA-HQ identities with 12 reference images each. We compute ISM as the cosine similarity between the normalized ArcFace~\citep{deng2019arcface} embedding of the largest detected generated face and the normalized mean embedding of its reference images. Because ISM excludes images without a detected face, we jointly report FDFR, the fraction of generated images in which no face is detected. We also report DINO-I, CLIP-I, and BRISQUE, but omit CLIP-T because all methods use the same fixed face prompt. Metrics are 
averaged over generation seeds and 
across concepts or identities. For general data synthesis, we follow the train-on-synthetic-test-on-real (TSTR) protocol: a classifier is trained exclusively on class-balanced synthetic images and evaluated on the official real test set. We report test accuracy as the primary metric and provide qualitative samples in the appendix.

\subsection{Few-Shot Personalized Generation}

\begin{figure}[t]
    \centering
    \includegraphics[width=0.8\linewidth]{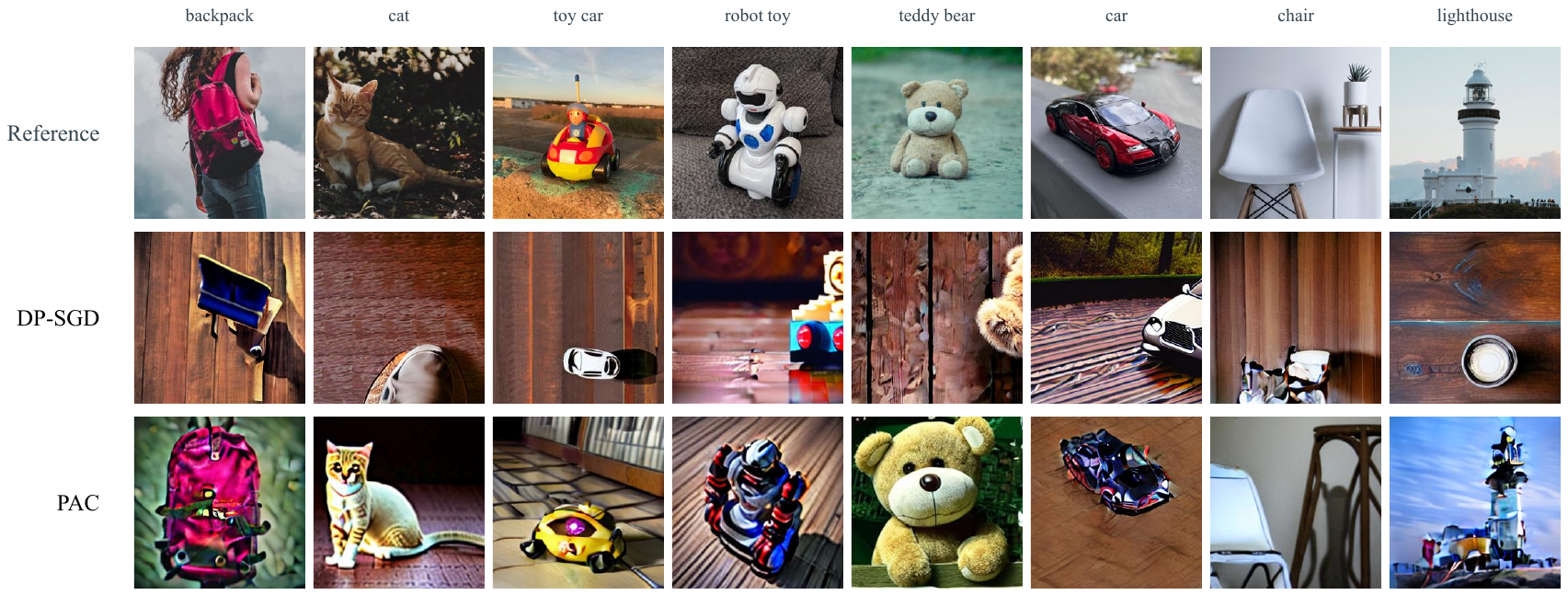}
    \caption{Qualitative comparison of synthetic images generated by PAC-privatized and DP-SGD-trained LoRA adapters at the matched privacy level $\varepsilon=10$ and $\delta=10^{-5}$. 
    }
    \label{fig:example_concept}
\end{figure}

\noindent\textbf{LoRA.} We evaluate LoRA adaptation
on DreamBooth and CustomConcept101, comparing DP-SGD and PAC. Figure~\ref{fig:example_concept} presents qualitative examples of concepts drawn from these datasets at a matched privacy level. We can see that DP-SGD often fails to acquire conceptual information from the reference images. For example, the generated images for \emph{bag} and \emph{lighthouse} do not capture the corresponding concepts. For \emph{teddy bear}, DP-SGD reproduces a fluffy texture but fails to preserve the overall shape. 
In contrast, PAC
better preserves the distinctive attributes learned by the clean adapter. For example, for \emph{robot}, PAC retains the overall appearance, 
and generated examples for \emph{teddy bear} and \emph{bag} also preserve visual characteristics that clearly resemble their reference images. 

\begin{table}[ht]
\centering
\caption{Comparison of PAC and DP-SGD for few-shot personalized generation. Results are averaged over 15 releases and then macro-averaged across subjects.}
\label{tab:fewshot}
\small
\setlength{\tabcolsep}{5pt}
\begin{tabular}{lc cc cc cc cc}
\hline
\multirow{2}{*}{Dataset}
& \multirow{2}{*}{$\varepsilon$}
& \multicolumn{2}{c}{DINO-I $\uparrow$}
& \multicolumn{2}{c}{CLIP-I $\uparrow$}
& \multicolumn{2}{c}{CLIP-T $\uparrow$}
& \multicolumn{2}{c}{BRISQUE $\downarrow$} \\
& & PAC & DP
  & PAC & DP
  & PAC & DP
  & PAC & DP \\
\hline

\multirow{5}{*}{DreamBooth}
& 3
& \textbf{.232} & .140
& \textbf{.596} & .577
& .212 & \textbf{.232}
& 30.2 & \textbf{29.4} \\

& 5
& \textbf{.306} & .141
& \textbf{.627} & .577
& .232 & \textbf{.233}
& \textbf{27.7} & 28.0 \\

& 10
& \textbf{.352} & .149
& \textbf{.645} & .579
& \textbf{.245} & .235
& 26.2 & \textbf{25.6} \\

& 20
& \textbf{.373} & .158
& \textbf{.655} & .581
& \textbf{.252} & .240
& 25.4 & \textbf{25.0} \\

& 50
& \textbf{.379} & .179
& \textbf{.657} & .589
& \textbf{.254} & .247
& 25.1 & \textbf{21.8} \\
\hline

\multirow{5}{*}{CustomConcept101}
& 3
& \textbf{.199} & .131
& \textbf{.561} & .542
& .200 & \textbf{.233}
& 31.7 & \textbf{27.5} \\
& 5
& \textbf{.286} & .129
& \textbf{.596} & .543
& .221 & \textbf{.232}
& \textbf{27.8} & \textbf{27.8} \\
& 10
& \textbf{.353} & .141
& \textbf{.625} & .546
& \textbf{.236} & \textbf{.236}
& 25.3 & \textbf{24.5} \\
& 20
& \textbf{.383} & .152
& \textbf{.639} & .550
& \textbf{.244} & .240
& \textbf{24.2} & 24.4 \\
& 50
& \textbf{.390} & .174
& \textbf{.642} & .557
& \textbf{.247} & \textbf{.247}
& 23.9 & \textbf{22.0} \\
\hline
\end{tabular}
\end{table}
Table~\ref{tab:fewshot} presents the quantitative comparison. On both datasets, PAC achieves substantially higher concept fidelity across all evaluated privacy budgets. On DreamBooth, PAC improves DINO-I from 0.140 to 0.232 at $\varepsilon=3$, from 0.149 to 0.352 at $\varepsilon=10$, and from 0.179 to 0.379 at $\varepsilon=50$. The corresponding CLIP-I scores increase from 0.577 to 0.596, from 0.579 to 0.645, and from 0.589 to 0.657, respectively. A similar trend appears on CustomConcept101. PAC improves DINO-I from 0.131 to 0.199 at $\varepsilon=3$, from 0.141 to 0.353 at $\varepsilon=10$, and from 0.174 to 0.390 at $\varepsilon=50$. It also improves CLIP-I from 0.542 to 0.561, from 0.546 to 0.625, and from 0.557 to 0.642 at the same budgets. PAC sacrifices some text alignment under tight privacy constraints, but this gap closes as the budget increases: on DreamBooth, PAC surpasses DP-SGD in CLIP-T from $\varepsilon=10$ onward, while on CustomConcept101 the two methods are comparable at larger budgets. BRISQUE does not exhibit a consistent winner, indicating that PAC's advantage primarily comes from preserving concept-specific information rather than uniformly improving perceptual quality. Overall, PAC retains substantially more of the learned concept than DP-SGD under matched privacy budgets.

\begin{figure}[ht]
    \centering
    \includegraphics[width=0.8\linewidth]{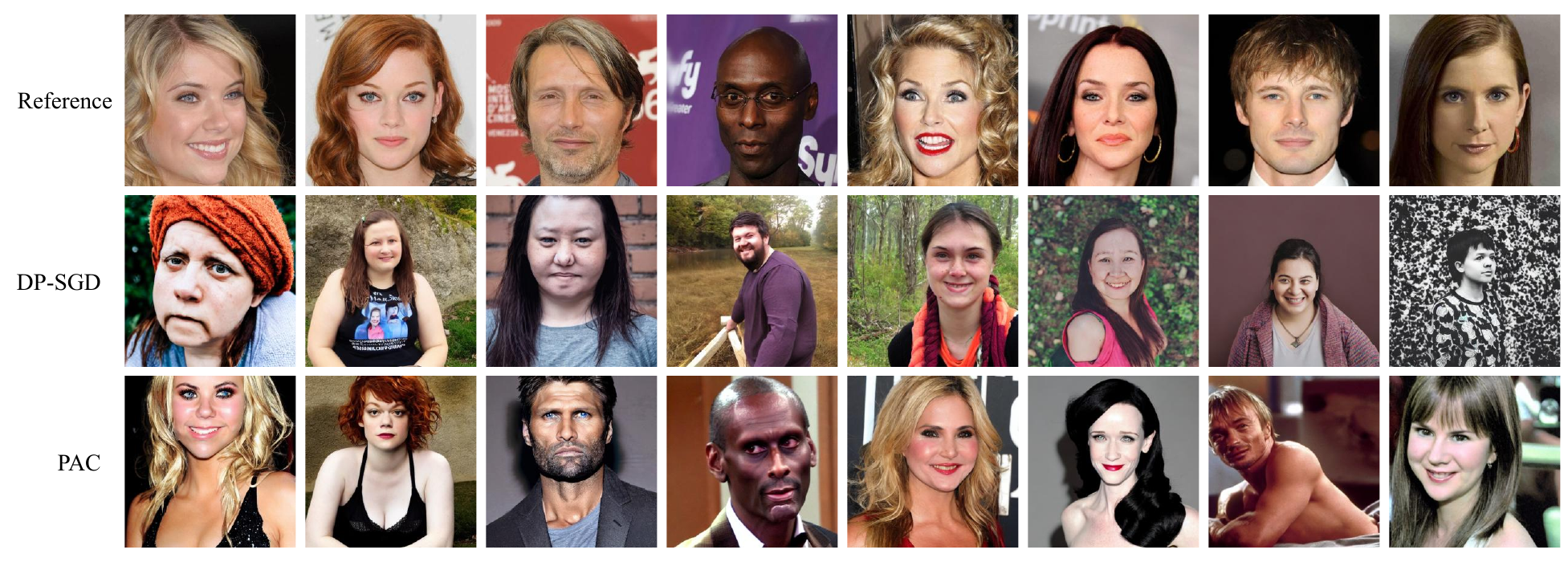}
    \caption{Qualitative comparison of private personalized generation via TI on CelebA-HQ at $\varepsilon=10$ and $\delta=10^{-5}$. 
    }
    \label{fig:example_face}
\end{figure}

\begin{table}[ht]
\centering
\caption{Private Textual Inversion results on CelebA-HQ. Results are averaged over 50 identities. All private mechanisms use $\delta=10^{-5}$.}
\label{tab:celeba_ti}
\small
\begin{tabular}{c cc cc cc cc cc}
\hline
\multirow{2}{*}{$\varepsilon$}
& \multicolumn{2}{c}{ISM $\uparrow$}
& \multicolumn{2}{c}{DINO-I $\uparrow$}
& \multicolumn{2}{c}{CLIP-I $\uparrow$}
& \multicolumn{2}{c}{BRISQUE $\downarrow$}
& \multicolumn{2}{c}{FDFR $\downarrow$} \\
& PAC & DP
& PAC & DP
& PAC & DP
& PAC & DP
& PAC & DP \\
\hline
3
& \textbf{.037} & .002
& \textbf{.301} & .197
& \textbf{.522} & .471
& \textbf{7.1} & \textbf{7.1}
& .24 & \textbf{.22} \\
5
& \textbf{.057} & .002
& \textbf{.347} & .198
& \textbf{.554} & .471
& \textbf{6.5} & 7.0
& .23 & \textbf{.22} \\
10
& \textbf{.103} & .002
& \textbf{.409} & .200
& \textbf{.603} & .471
& \textbf{6.0} & 6.9
& .23 & \textbf{.22} \\
20
& \textbf{.165} & .003
& \textbf{.458} & .205
& \textbf{.650} & .472
& \textbf{5.4} & 6.8
& \textbf{.21} & .22 \\
50
& \textbf{.195} & .004
& \textbf{.478} & .214
& \textbf{.670} & .474
& \textbf{5.3} & 6.9
& \textbf{.21} & .22 \\
\hline
\end{tabular}
\end{table}

\noindent\textbf{Textual Inversion.} We compare the performance of TI privatized using DP-SGD and PAC on CelebA-HQ. Figure~\ref{fig:example_face} presents qualitative results. DP-SGD largely fails to capture identity-specific facial characteristics, whereas
PAC preserves substantially more facial structure and identity information.
The quantitative results in Table~\ref{tab:celeba_ti} support this observation. Even at $\varepsilon=3$, PAC improves ISM from 0.002 to 0.037, DINO-I from 0.197 to 0.301, and CLIP-I from 0.471 to 0.522 compared with DP-SGD. As the privacy budget increases, PAC improves steadily, reaching 0.195 ISM, 0.478 DINO-I, and 0.670 CLIP-I at $\varepsilon=50$, whereas DP-SGD remains nearly unchanged across all budgets. PAC also achieves better perceptual quality for $\varepsilon\geq5$, while the two methods exhibit comparable face-detection failure rates. Overall, final-output PAC perturbation preserves personalized identity considerably better than iterative DP-SGD perturbation for Textual Inversion.

\begin{table}[ht]
\centering
\caption{TSTR test accuracy (\%) across privacy budgets. Each entry reports the mean $\pm$ the half-width of the 95\% confidence interval.}
\label{tab:general_synthesis}
\small
\setlength{\tabcolsep}{6pt}
\begin{tabular}{llccccc}
\hline
Dataset & Method
& $\varepsilon=3$
& $\varepsilon=5$
& $\varepsilon=10$
& $\varepsilon=20$
& $\varepsilon=50$ \\
\hline

\multirow{2}{*}{FMNIST}
& DP
& $35.7\pm4.9$
& $36.1\pm5.9$
& $37.2\pm6.5$
& $39.4\pm4.4$
& $39.9\pm6.9$ \\
& PAC
& $\bm{36.4\pm4.8}$
& $\bm{37.7\pm4.5}$
& $\bm{43.1\pm4.6}$
& $\bm{50.9\pm5.6}$
& $\bm{52.6\pm4.2}$ \\
\hline

\multirow{2}{*}{EuroSAT}
& DP
& $33.1\pm2.9$
& $34.3\pm2.8$
& $35.5\pm2.0$
& $35.8\pm2.2$
& $37.6\pm3.2$ \\
& PAC
& $\bm{33.4\pm3.0}$
& $\bm{36.6\pm3.9}$
& $\bm{39.6\pm3.8}$
& $\bm{44.1\pm3.4}$
& $\bm{46.4\pm3.8}$ \\
\hline

\multirow{2}{*}{BloodMNIST}
& DP
& $12.3\pm3.2$
& $13.3\pm3.5$
& $14.1\pm3.8$
& $14.5\pm3.7$
& $14.8\pm2.8$ \\
& PAC
& $\bm{12.7\pm1.6}$
& $\bm{18.5\pm2.5}$
& $\bm{26.4\pm5.8}$
& $\bm{28.6\pm6.8}$
& $\bm{30.2\pm5.0}$ \\
\hline
\end{tabular}
\end{table}

\subsection{General Data Synthesis}

We further evaluate PAC-private LoRA adaptation for general data synthesis. 
Table~\ref{tab:general_synthesis} reports the resulting test accuracy with 95\% confidence intervals. 
Both PAC and DP-SGD reduce utility because they introduce noise into the learned model. However, compared with DP-SGD, PAC preserves substantially more of the utility provided by the clean adapter. At $\varepsilon=50$, which corresponds to a very loose privacy budget, the synthetic data generated after PAC perturbation achieve test accuracies of 52.6\% on FashionMNIST, 46.4\% on EuroSAT, and 30.2\% on BloodMNIST, approaching the corresponding clean LoRA ceilings. In contrast, DP-SGD yields average downstream accuracies below 40\% on FashionMNIST and EuroSAT and only 14.8\% on BloodMNIST. When the budget decreases to $\varepsilon=10$, PAC still exceeds DP-SGD by 5.9\%, 4.1\%, and 12.3\%  on the three datasets, respectively. Under tighter privacy budgets, both PAC and DP-SGD struggle to preserve the utility of the clean adapter, although PAC continues to achieve slightly higher average accuracy.

\begin{table}[h]
\centering
\caption{Effect of LoRA canonicalization on FashionMNIST. Each entry reports TSTR test accuracy (\%) as the mean $\pm$ the half-width of the 95\% confidence interval.}
\label{tab:canonicalization_ablation}
\small
\setlength{\tabcolsep}{6pt}
\begin{tabular}{lccccc}
\hline
Mechanism & $\varepsilon=3$ & $\varepsilon=5$ & $\varepsilon=10$ & $\varepsilon=20$ & $\varepsilon=50$ \\
\hline
Raw factors
& $34.9\pm4.4$
& $35.6\pm5.8$
& $40.6\pm4.3$
& $48.3\pm4.4$
& $50.6\pm4.6$ \\
Canonicalized factors
& \bm{$36.4\pm4.8$}
& \bm{$37.7\pm4.5$}
& \bm{$43.1\pm4.6$}
& \bm{$50.9\pm5.6$}
& \bm{$52.6\pm4.2$} \\
\hline
\end{tabular}
\end{table}

\subsection{Effect of LoRA Canonicalization}

Finally, we evaluate the effect of the
LoRA canonicalization step 
on PAC privatization. The raw-factor and canonicalized variants differ only in whether canonicalization is applied before covariance estimation and noise perturbation. Table~\ref{tab:canonicalization_ablation} shows that 
canonicalization improves TSTR accuracy by 1.5, 2.1, 2.5, 2.6, and 2.0 percentage points at $\varepsilon\in\{3,5,10,20,50\}$, respectively. Although the confidence intervals overlap at individual privacy budgets, the canonicalized variant achieves a higher mean accuracy at every operating point. This consistent trend indicates that eliminating the arbitrary scale of equivalent LoRA factorizations improves the stability and utility of PAC noise calibration.

%% file: sections/6.conclusion.tex
\section{Conclusion}
In this work, we propose an alternative to DP-SGD for private diffusion generation by perturbing the learned output once at the end of training, thereby avoiding privacy composition across optimization steps. We consider a distribution reconstruction privacy notion and derive a reconstruction-based conversion that maps $(\varepsilon,\delta)$-DP parameters to a comparable mutual-information budget. Given this budget, we repeatedly execute the training mechanism, estimate the covariance of its outputs, and calibrate an anisotropic noise distribution for the final release. To make this procedure computationally feasible for diffusion models, we instantiate it with two low-dimensional adaptation mechanisms: LoRA and Textual Inversion. For LoRA, we further introduce a function-preserving canonicalization that removes factor-scale ambiguity and improves calibration stability. We evaluate the framework on both general-purpose data synthesis and few-shot personalized generation. The results show that the proposed approach preserves substantially more downstream and concept-specific utility than DP-SGD across a broad range of matched privacy budgets.

\section*{AI Use Statement}
We used generative AI tools to improve the clarity and grammar of the manuscript and to identify typographical errors. We also used these tools to assist with code debugging, LaTeX formatting, and the presentation of figures and tables. The synthetic image datasets evaluated in this work were generated by the diffusion models studied in the paper as part of our experimental methodology, rather than by an external AI assistant. We did not use generative AI tools for research problem formulation or ideation. The authors reviewed all AI-assisted content and verified the assisted code through executed experiments. All reported numerical results were obtained from these experiments and were not generated by an AI assistant. We take full responsibility for the final content of this work, including any content produced with the assistance of generative AI tools.

\section*{Ethics Statement}
This work studies privacy-preserving image synthesis and is intended to reduce the risk of reconstructing sensitive training data from released generative models or synthetic images. Our experiments use publicly available benchmark datasets, and we do not collect new human-subject data. CelebA-HQ contains identifiable facial images and is used only for evaluating privacy-preserving personalization under the established benchmark protocol. 

\section*{Reproducibility Statement}
We provide the complete PAC calibration procedure in Algorithm~\ref{alg:diffusion_pac_calibration}. The appendix documents the model configurations, optimization settings, generation settings, and evaluation metrics. The supplementary material includes the core implementation of our method.

%% file: sections/appendixa.tex
\section{Additional Theoretical Background}

\subsection{\texorpdfstring{$f$-Differential Privacy}{f-Differential Privacy}}

We briefly review $f$-differential privacy ($f$-DP), which characterizes privacy through binary hypothesis testing between the outputs of a mechanism on neighboring datasets~\cite{dong2022gaussian}.

\begin{definition}[Trade-off Function]
Let $P$ and $Q$ be two distributions over the same output space. For a hypothesis test $\psi$, let $\alpha_\psi=\mathbb E_P[\psi]$ and $\beta_\psi=1-\mathbb E_Q[\psi]$ denote its type-I and type-II errors. The trade-off function between $P$ and $Q$ is
\begin{equation}
T(P,Q)(\alpha):=\inf_{\psi:\,\alpha_\psi\leq\alpha}\beta_\psi,
\qquad \alpha\in[0,1].
\end{equation}
\end{definition}

\begin{definition}[$f$-Differential Privacy~\cite{dong2022gaussian}]
Let $f:[0,1]\rightarrow[0,1]$ be a trade-off function. A randomized mechanism $M$ satisfies $f$-DP if, for every pair of neighboring datasets $D\simeq D'$,
\begin{equation}
T(M(D),M(D'))(\alpha)\geq f(\alpha),
\qquad \forall\alpha\in[0,1].
\end{equation}
\end{definition}

The standard $(\varepsilon,\delta)$-DP guarantee is equivalent to $f_{\varepsilon,\delta}$-DP, where
\begin{equation}
f_{\varepsilon,\delta}(\alpha)
=
\max\{0,\,
1-\delta-e^\varepsilon\alpha,\,
e^{-\varepsilon}(1-\delta-\alpha)\}.
\label{eq:eps_delta_fdp}
\end{equation}

\subsection{Reconstruction Advantage Provided by DP}

\citet{kulynych2025unifying} use strong reconstruction robustness (SRR) to interpret differential privacy in terms of reconstruction risk. SRR considers a strong adversary that knows all but one record of the input dataset and attempts to reconstruct the remaining record from the mechanism output. We show that the resulting reconstruction-advantage bound also applies to the weaker adversary considered in our setting, who observes only the mechanism output.

\begin{definition}[Datapoint Reconstruction Advantage~\citep{kulynych2025unifying}]
\label{def:datapoint_reconstruction}
Let $\mathcal X$ be the data domain, $\mathcal D$ a distribution over $\mathcal X$, and $M:\mathcal X^n\rightarrow\Theta$ a randomized mechanism. Let $\bar D\in\mathcal X^{n-1}$ be a fixed partial dataset known to the adversary, and let $Z\sim\mathcal D$ denote the remaining target record. Given a distance function $d:\mathcal X\times\mathcal X\rightarrow\mathbb R$ and a reconstruction radius $r$, define the optimal prior reconstruction error as
\begin{equation}
\delta_p^{o}:=1-\sup_{\widehat z\in\mathcal X}\Pr_{Z\sim\mathcal D}\bigl[d(Z,\widehat z)\leq r\bigr].
\end{equation}

Let $O\sim M(\bar D\cup\{Z\})$ be the released mechanism output. After observing both $\bar D$ and $O$, the optimal posterior reconstruction
error is
\begin{equation}
\delta_p(M\mid\bar D):=1-\sup_{\mathcal A} \Pr_{\substack{
Z\sim\mathcal D\\
O\sim M(\bar D\cup\{Z\})
}}\bigl[d\bigl(Z,\mathcal A(\bar D,O)\bigr)\leq r\bigr],
\end{equation}
where
$\mathcal A:\mathcal X^{n-1}\times\Theta\rightarrow\mathcal X$ is a reconstruction attack. The SRR reconstruction advantage is
\begin{equation}
\operatorname{Adv}
(M,\bar D,\mathcal D;d,r):=\delta_p^{o}-\delta_p(M\mid\bar D).
\end{equation}
\end{definition}

The datapoint reconstruction advantage can be bounded using $f$-DP.
\begin{theorem}[Reconstruction Advantage under $(\varepsilon,\delta)$-DP~\citep{kulynych2025unifying}]
\label{thm:fdp_reconstruction}
If $M$ satisfies $(\varepsilon,\delta)$-DP, then, for every
partial dataset $\bar D$,
\begin{equation}
\operatorname{Adv}
(M,\bar D,\mathcal D;d,r)
\leq
\delta_p^o
-
f_{\varepsilon,\delta}(1-\delta_p^o).
\end{equation}
Equivalently,
\begin{equation}
\delta_p(M\mid\bar D)
\geq
f_{\varepsilon,\delta}(1-\delta_p^o).
\end{equation}
\end{theorem}

\begin{corollary}[Output-Only Reconstruction Advantage]
\label{cor:output_reconstruction}
Let $\bar D\sim\mathcal D^{n-1}$, and let $\delta_p(M)$ denote the optimal reconstruction error of an adversary that observes only the mechanism output. If $M$ satisfies $(\varepsilon,\delta)$-DP, then
\begin{equation}
\delta_p^o-\delta_p(M)
\leq
\delta_p^o
-
f_{\varepsilon,\delta}(1-\delta_p^o).
\label{eq:dp_reconstruction_advantage}
\end{equation}
\end{corollary}

\begin{proof}
An adversary observing only the mechanism output is no stronger than one observing both the output and the partial dataset. More precisely, for any output-only attack $\mathcal A_0$, define the corresponding informed attack by
$\mathcal A(\bar D,O)=\mathcal A_0(O)$. Therefore,
\begin{equation}
\begin{aligned}
1-\delta_p(M)
&\leq
\int
\sup_{\mathcal A}
\Pr\bigl[
d(Z,\mathcal A(\bar d,O))\leq r
\mid \bar D=\bar d
\bigr]
p(\bar d)\,d\bar d \\
&=
1-
\int
\delta_p(M\mid\bar D=\bar d)
p(\bar d)\,d\bar d.
\end{aligned}
\end{equation}
Equivalently,
\begin{equation}
\delta_p(M)
\geq
\int
\delta_p(M\mid\bar D=\bar d)
p(\bar d)\,d\bar d.
\end{equation}
By Theorem~\ref{thm:fdp_reconstruction}, $\delta_p(M\mid\bar D=\bar d)
\geq
f_{\varepsilon,\delta}(1-\delta_p^o)$ for every $\bar d$. Hence, $\delta_p(M)
\geq
f_{\varepsilon,\delta}(1-\delta_p^o).$ Subtracting both sides from $\delta_p^o$ proves the result.
\end{proof}
Let $\alpha:=1-\delta_p^o$ denote the optimal prior reconstruction success rate. Corollary~\ref{cor:output_reconstruction} implies $1-\delta_p(M)\leq 1-f_{\varepsilon,\delta}(\alpha)$. For any fixed $\alpha$ satisfying $f_{\varepsilon,\delta}(\alpha)>0$, we define $\beta(\alpha):=D_{\mathrm{KL}}(\operatorname{Bern}(1-f_{\varepsilon,\delta}(\alpha)) \,\|\,\operatorname{Bern}(\alpha))$. Because the KL reconstruction advantage is bounded by mutual information, $I(X;M(X))\leq\beta(\alpha)$ gives the same upper bound on posterior reconstruction success as $(\varepsilon,\delta)$-DP at the fixed prior success rate $\alpha$.

%% file: sections/appendixb.tex
\section{Experimental Details}
\label{app:experimental_details}

\subsection{Training Configurations}
Table~\ref{tab:training_details} summarizes the main training configurations. All experiments use the AdamW optimizer. PAC and DP-SGD use identical pretrained models, initialization, adaptation architecture, optimization schedule, preprocessing, and training prompts.

\begin{table}[ht]
\centering
\caption{Training configurations used in our experiments. Generation resolution matches training resolution.}
\label{tab:training_details}
\resizebox{\textwidth}{!}{
\begin{tabular}{lccccccc}
\hline
Dataset & Backbone & Resolution & Adaptation & Batch & Updates & Learning rate & Warm-up \\
\hline
DreamBooth & SD V-1.5 & $256\times 256$ & LoRA, $r=1$ & 4 & 250 & $2\times10^{-4}$ & 7  \\
CustomConcept101 & SD V-1.5 & $256\times 256$ & LoRA, $r=1$ & 4 & 250 & $2\times10^{-4}$ & 7 \\
CelebA-HQ & SD V-1.5 & $512\times512$ & Textual Inversion & 1 & 500 & $5\times10^{-4}$ & 0  \\
FashionMNIST & BK-SDM-Small & $256\times 256$ & LoRA, $r=1$ & 8 & 7,500 & $10^{-4}$ & 25  \\
EuroSAT & BK-SDM-Small & $256\times 256$ & LoRA, $r=1$ & 8 & 8,100 & $10^{-4}$ & 25  \\
BloodMNIST & BK-SDM-Small & $256\times 256$ & LoRA, $r=1$ & 8 & 7,475 & $10^{-4}$ & 25 \\
\hline
\end{tabular}
}
\end{table}

For DreamBooth and CustomConcept101, we train rank-one LoRA adapters on the key and value projections of the cross-attention layers. The private reference sets contain 4--6 images per DreamBooth subject and 3--15 images per CustomConcept101 subject. Subject identifiers are used only to locate the reference images and organize the outputs. The text condition always uses the associated class word. For example, \texttt{dog2} is trained with ``a photo of a dog,'' \texttt{pet\_cat3} with ``a photo of a cat,'' and \texttt{transport\_car10} with ``a photo of a car.'' Consequently, subjects belonging to the same class receive identical prompts and are distinguished only through their private reference images and learned LoRA adapters. For CelebA-HQ, we optimize one 768-dimensional Textual Inversion embedding using all 12 reference images of each identity. At every update, the training prompt is sampled uniformly from the eight standard Textual Inversion templates. For FashionMNIST, EuroSAT, and BloodMNIST, we use class-specific prompt templates tailored to the visual domain of each dataset. The details are provided in Table~\ref{tab:training_prompts}.
\begin{table}[!t]
\centering
\caption{Training prompts used in our experiments. For LoRA personalization, \{class\} denotes the semantic class specified by the dataset metadata rather than the subject-directory name.}
\label{tab:training_prompts}
\scriptsize
\setlength{\tabcolsep}{6pt}
\renewcommand{\arraystretch}{1.15}
\begin{tabular}{p{3.0cm} p{8.0cm}}
\hline
Dataset & Training prompts \\
\hline

DreamBooth
&
\texttt{a photo of a \{class\}}
\\

CustomConcept101
&
\texttt{a photo of a \{class\}}
\\

CelebA-HQ
&
\begin{tabular}[t]{@{}l@{}}
\texttt{a photo of sks person} \\
\texttt{a rendering of sks person} \\
\texttt{a cropped photo of sks person} \\
\texttt{the photo of sks person} \\
\texttt{a close-up photo of sks person} \\
\texttt{a bright photo of sks person} \\
\texttt{a good photo of sks person} \\
\texttt{a photo of one sks person}
\end{tabular}
\\
FashionMNIST
&
\begin{tabular}[t]{@{}l@{}}
\textbf{Template:} \texttt{a centered grayscale product image of \{X\}} \\
\phantom{\textbf{Template:} }\texttt{on a plain black background} \\
\textbf{Class descriptions:} \\
\texttt{a short-sleeved T-shirt} \\
\texttt{a pair of trousers} \\
\texttt{a pullover sweater} \\
\texttt{a dress} \\
\texttt{a coat} \\
\texttt{a sandal} \\
\texttt{a long-sleeved shirt} \\
\texttt{a sneaker} \\
\texttt{a bag} \\
\texttt{an ankle boot}
\end{tabular}
\\

EuroSAT
&
\begin{tabular}[t]{@{}l@{}}
\textbf{Template:} \texttt{an overhead satellite image of \{X\}} \\
\textbf{Class descriptions:} \\
\texttt{annual cropland} \\
\texttt{a forest} \\
\texttt{herbaceous vegetation} \\
\texttt{a highway} \\
\texttt{an industrial area} \\
\texttt{pasture land} \\
\texttt{permanent cropland} \\
\texttt{a residential area} \\
\texttt{a river} \\
\texttt{a sea or lake}
\end{tabular}
\\

BloodMNIST
&
\begin{tabular}[t]{@{}l@{}}
\textbf{Template:} \texttt{a microscopy image of \{X\}} \\
\textbf{Class descriptions:} \\
\texttt{a basophil blood cell} \\
\texttt{an eosinophil blood cell} \\
\texttt{an erythroblast} \\
\texttt{an immature granulocyte} \\
\texttt{a lymphocyte blood cell} \\
\texttt{a monocyte blood cell} \\
\texttt{a neutrophil blood cell} \\
\texttt{a platelet}
\end{tabular}
\\
\hline
\end{tabular}
\end{table}

\paragraph{Privacy mechanisms.}
DP-SGD clips each per-example gradient to $C=0.01$. Given a private dataset of size $n$ and a minibatch of size $b$, the sampling rate is $q=\min\{b/n,1\}$. We use the PRV accountant~\cite{gopi2021numerical} to determine the noise multiplier for each target $\varepsilon$, with $\delta=10^{-5}$. For PAC calibration, we repeat the complete training procedure $K=96$ times to estimate the output covariance. At $K=96$, the relative change in the trace of the estimated covariance falls below the convergence threshold $\tau=0.05$.

\subsection{Evaluation Settings}

\begin{table}[!t]
\centering
\caption{Generation and evaluation prompts used in our experiments. For object personalization, \{class\} is replaced by the semantic class associated with each subject; no identifier token is used. General-synthesis prompts are identical to their training prompts in Table~\ref{tab:training_prompts}.}
\label{tab:evaluation_prompts}
\scriptsize
\setlength{\tabcolsep}{1pt}
\begin{tabular}{p{3.2cm} p{9.6cm}}
\hline
Dataset & Generation and evaluation prompts \\
\hline

DreamBooth and CustomConcept101
&
\begin{tabular}[t]{@{}l@{}}
\texttt{a \{class\} in the jungle} \\
\texttt{a \{class\} in the snow} \\
\texttt{a \{class\} on the beach} \\
\texttt{a \{class\} on a cobblestone street} \\
\texttt{a \{class\} on top of pink fabric} \\
\texttt{a \{class\} on top of a wooden floor} \\
\texttt{a \{class\} with a city in the background} \\
\texttt{a \{class\} with a mountain in the background} \\
\texttt{a \{class\} with the Eiffel Tower in the background} \\
\texttt{a \{class\} floating on top of water}
\end{tabular}
\\

CelebA-HQ
&
\begin{tabular}[t]{@{}l@{}}
\textbf{Generation:} \texttt{a photo of sks person} \\
\textbf{Class-only reference:} \texttt{a photo of a person}
\end{tabular}
\\

FashionMNIST
&
Same as the class-specific FashionMNIST training prompts in
Table~\ref{tab:training_prompts}.
\\

EuroSAT
&
Same as the class-specific EuroSAT training prompts in
Table~\ref{tab:training_prompts}.
\\

BloodMNIST
&
Same as the class-specific BloodMNIST training prompts in
Table~\ref{tab:training_prompts}.
\\

\hline
\end{tabular}
\end{table}

\paragraph{Generation protocol.}
Table~\ref{tab:evaluation_prompts} presents the evaluation prompts used for each dataset. For DreamBooth and CustomConcept101, each released adapter is evaluated using ten fixed DreamBooth prompt templates instantiated with the class name only and without an identifier token. We generate four images per prompt using 50 denoising steps, classifier-free guidance of 7.5, and a resolution of $256\times256$. For CelebA-HQ, we use the single fixed prompt ``a photo of sks person'' and generate 16 images per identity using 50 denoising steps, classifier-free guidance of 7.5, and a resolution of $512\times512$. For FashionMNIST, EuroSAT, and BloodMNIST, we generate class-balanced synthetic training sets using the same class-specific prompts employed during training. Generation uses 25 denoising steps, classifier-free guidance of 7.5, and a resolution of $256\times256$. Table~\ref{tab:synthesis_evaluation} summarizes the synthetic training-set sizes and real evaluation splits.

\begin{table}[ht]
\centering
\caption{Generation and downstream evaluation settings for general data synthesis.}
\label{tab:synthesis_evaluation}
\begin{tabular}{lccc}
\hline
Setting & FashionMNIST & EuroSAT & BloodMNIST \\
\hline
Synthetic images per class & 500 & 2,000 & 1,500 \\
Total synthetic images & 5,000 & 20,000 & 12,000 \\
Denoising steps & 25 & 25 & 25 \\
Guidance scale & 7.5 & 7.5 & 7.5 \\
Generation resolution & $256\times 256$ & $256\times 256$ & $256\times 256$ \\
Real test examples & 10,000 & 2,700 & 3,421 \\
Evaluation split & Official test & Held-out test & Official test \\
\hline
\end{tabular}
\end{table}

\paragraph{Object-personalization metrics.}
We evaluate subject fidelity using DINO-I and CLIP-I. DINO-I uses \texttt{dinov2\_vits14}, while CLIP-I uses the image encoder of CLIP ViT-L/14. For each generated image, we compute its cosine similarity to every reference image. Similarities are first averaged within each prompt and then across prompts. CLIP-T is the cosine similarity between the CLIP ViT-L/14 embeddings of a generated image and its generation prompt. Perceptual image quality is measured using BRISQUE, for which lower values indicate better quality. 

\paragraph{Face-personalization metrics.}
For facial identity evaluation, we specifically adopt identity score matching (ISM) to measure identity preservation. ISM is computed using InsightFace \texttt{buffalo\_l}. Faces are detected using SCRFD at $640\times640$, and identity embeddings are extracted using its ArcFace ResNet-50 model trained on WebFace600K. We average and normalize the ArcFace embeddings of the 12 reference images to obtain one reference identity vector. For each generated image, we retain the largest detected face and compute its cosine similarity to the reference vector. ISM is averaged only over images in which a face is detected. We therefore jointly report the face-detection failure rate (FDFR), defined as the fraction of generated images containing no detected face.

\paragraph{TSTR evaluation.}
For general data synthesis, we follow the train-on-synthetic-test-on-real protocol. A ResNet-50 initialized with ImageNet-1K V2 weights is trained exclusively on the generated images after replacing its final classification layer. The classifier is optimized for 30 epochs using SGD with a learning rate of 0.01, momentum 0.9, weight decay $10^{-4}$, a cosine learning-rate schedule, and a batch size of 128. Images are resized to $224\times224$, randomly flipped horizontally during training, and normalized using ImageNet statistics. EuroSAT is split into 21,600 training, 2,700 validation, and 2,700 test images. We select the number of optimization steps for EuroSAT and BloodMNIST using only their validation sets. All results reported in the main paper are evaluated on the corresponding held-out test sets.

%% file: sections/appendixc.tex
\section{Additional Qualitative Results}
\begin{figure}
    \centering
    \includegraphics[width=0.8\linewidth]{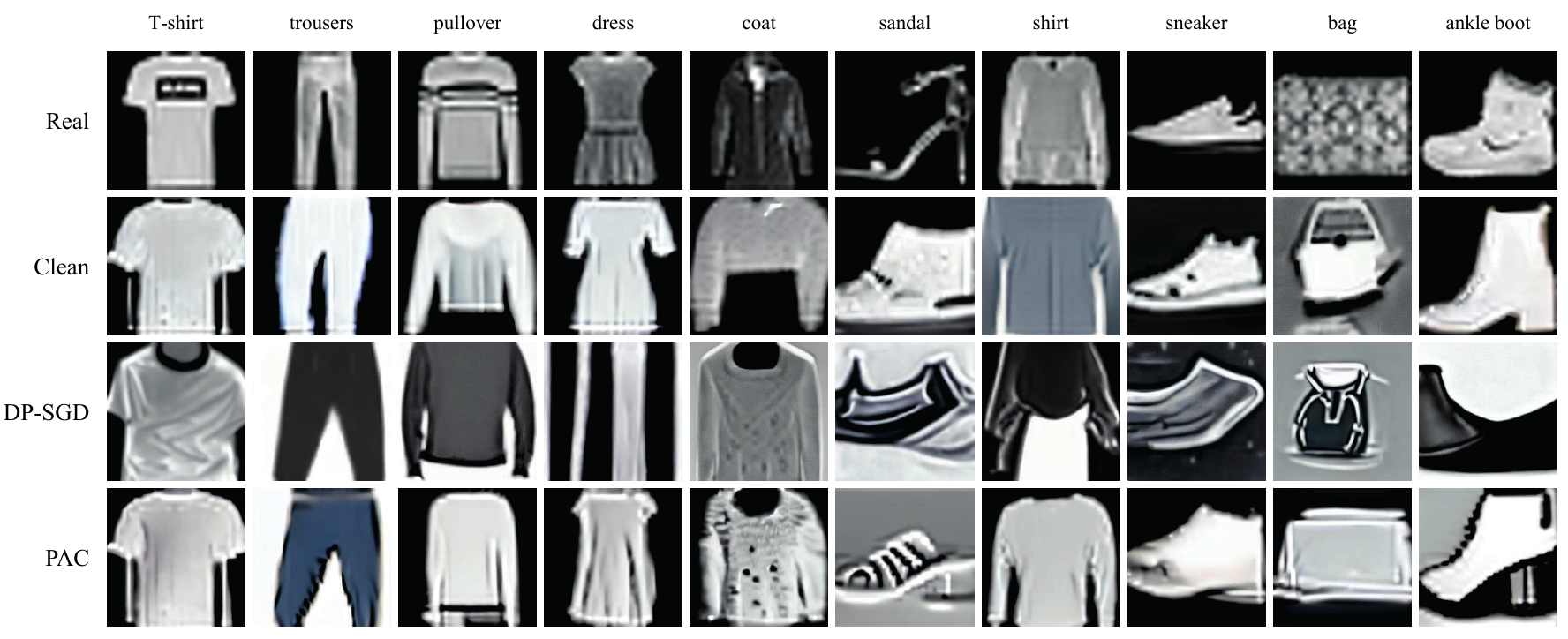}
    \caption{Qualitative comparison of synthetic images generated using PAC privatization and DP-SGD on FashionMNIST at $\varepsilon=10$ and $\delta=10^{-5}$.}
    \label{fig:example_fmnist}
\end{figure}

\begin{figure}
    \centering
    \includegraphics[width=0.8\linewidth]{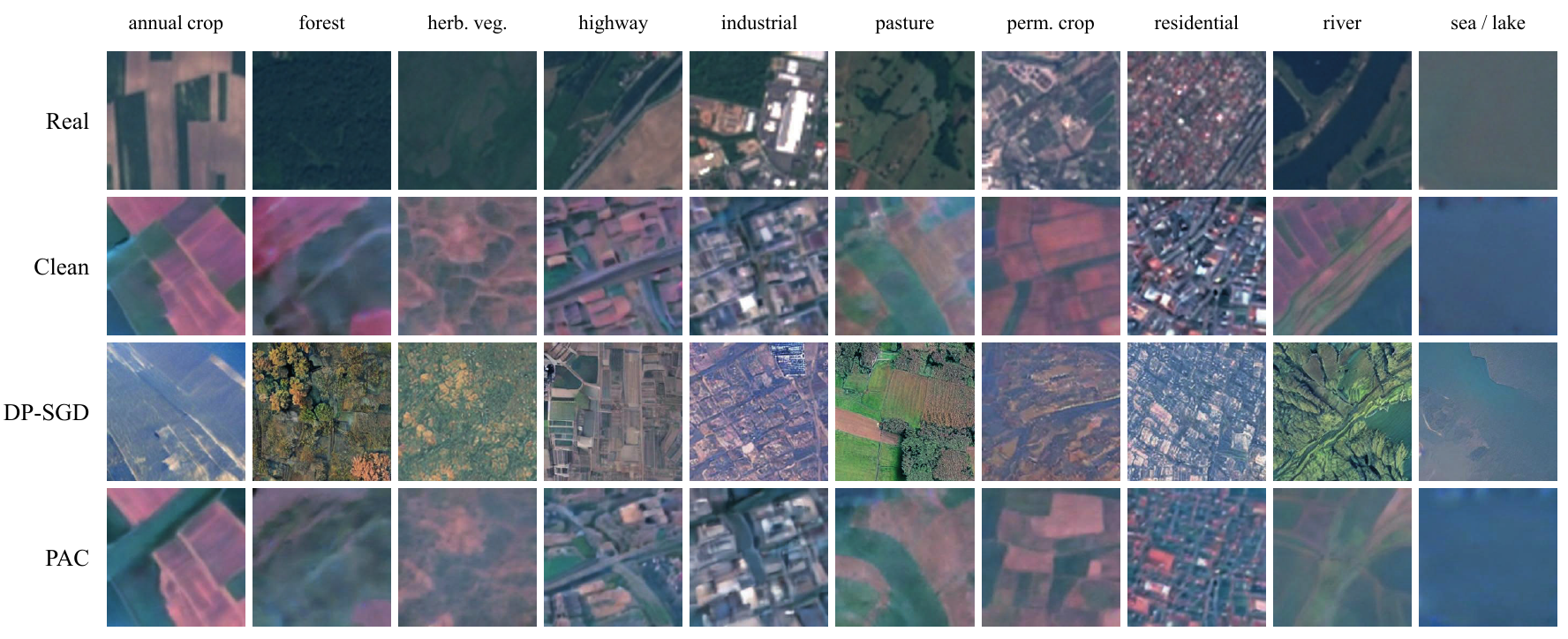}
    \caption{Qualitative comparison of synthetic images generated using PAC privatization and DP-SGD on EuroSAT at $\varepsilon=10$ and $\delta=10^{-5}$.}
    \label{fig:example_eurosat}
\end{figure}

\begin{figure}
    \centering
    \includegraphics[width=0.8\linewidth]{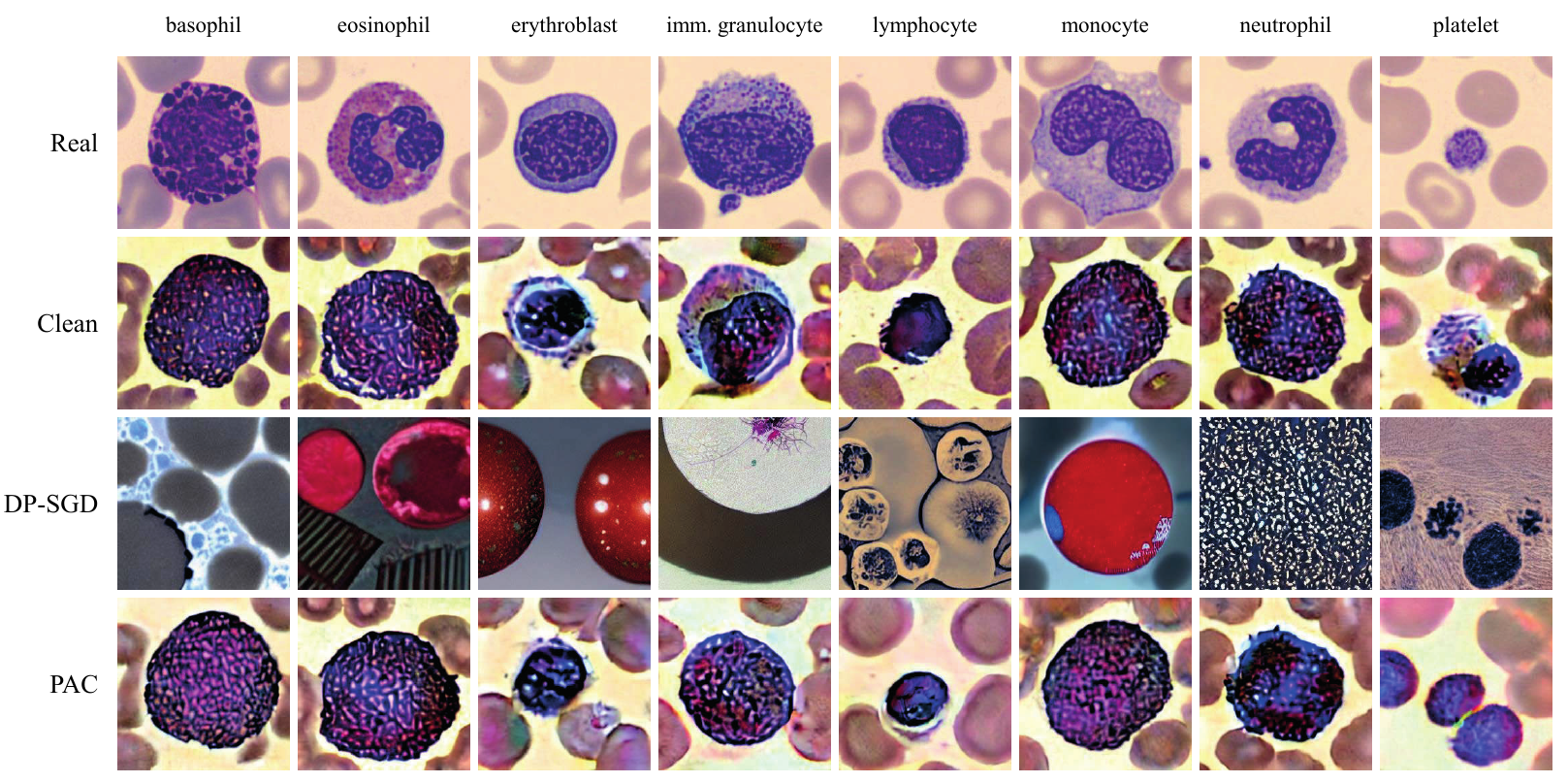}
    \caption{Qualitative comparison of synthetic images generated using PAC privatization and DP-SGD on BloodMNIST at $\varepsilon=10$ and $\delta=10^{-5}$.}
    \label{fig:example_bloodmnist}
\end{figure}

Figures~\ref{fig:example_fmnist}, \ref{fig:example_eurosat}, and \ref{fig:example_bloodmnist} present qualitative examples generated using PAC-private adaptation and DP-SGD on FashionMNIST, EuroSAT, and BloodMNIST, respectively, under the matched privacy budget $\varepsilon=10$ and $\delta=10^{-5}$. Compared with DP-SGD, PAC-private generation more clearly preserves the class-specific visual characteristics of the training data.

%% file: sections/appendixd.tex
\section{Additional Quantitative Results}

\begin{table}[ht]
\centering
\caption{TSTR test accuracy (\%) across privacy budgets. Each entry reports the mean $\pm$ the half-width of the 95\% confidence interval.}
\label{tab:general_synthesis_appendix}
\small
\setlength{\tabcolsep}{5pt}
\begin{tabular}{@{}llccccc@{}}
\hline
Dataset & Method
& $\varepsilon=3$
& $\varepsilon=4$
& $\varepsilon=5$
& $\varepsilon=7$
& $\varepsilon=10$ \\
\hline

\multirow{2}{*}{FMNIST}
& DP
& $35.7\pm4.9$
& $35.7\pm4.6$
& $36.1\pm5.9$
& $36.6\pm6.6$
& $37.2\pm6.5$ \\
& PAC
& $\bm{36.4\pm4.8}$
& $\bm{37.1\pm4.5}$
& $\bm{37.7\pm4.5}$
& $\bm{39.1\pm5.8}$
& $\bm{43.1\pm4.6}$ \\
\hline

\multirow{2}{*}{EuroSAT}
& DP
& $33.1\pm2.9$
& $34.3\pm2.9$
& $34.3\pm2.8$
& $34.4\pm2.6$
& $35.5\pm2.0$ \\
& PAC
& $\bm{33.4\pm3.0}$
& $\bm{34.8\pm3.4}$
& $\bm{36.6\pm3.9}$
& $\bm{39.4\pm3.2}$
& $\bm{39.6\pm3.8}$ \\
\hline

\multirow{2}{*}{BloodMNIST}
& DP
& $12.3\pm3.2$
& $13.2\pm3.4$
& $13.3\pm3.5$
& $13.8\pm3.6$
& $14.1\pm3.8$ \\
& PAC
& $\bm{12.7\pm1.6}$
& $\bm{15.9\pm2.6}$
& $\bm{18.5\pm2.5}$
& $\bm{19.3\pm3.5}$
& $\bm{26.4\pm5.8}$ \\
\hline
\end{tabular}
\end{table}

We additionally evaluate a finer range of privacy budgets with $\varepsilon\leq10$ to further compare DP-SGD and PAC privatization. Table~\ref{tab:general_synthesis_appendix} reports TSTR accuracy at $\varepsilon\in\{3,4,5,7,10\}$. PAC achieves higher mean accuracy than DP-SGD on all three datasets at every evaluated budget, with the performance gap generally increasing as $\varepsilon$ increases. The improvement is particularly pronounced on BloodMNIST, where DP-SGD remains nearly unchanged while PAC progressively recovers downstream utility.

\begin{table}[ht]
\centering
\caption{Private Textual Inversion results on 50 CelebA-HQ identities. Each entry reports the mean $\pm$ the half-width of the 95\% confidence interval. All private mechanisms use $\delta=10^{-5}$.}
\label{tab:celeba_ti_interval}
\small
\begin{tabular}{c cc cc cc}
\hline
\multirow{2}{*}{$\varepsilon$}
& \multicolumn{2}{c}{ISM $\uparrow$}
& \multicolumn{2}{c}{DINO-I $\uparrow$}
& \multicolumn{2}{c}{CLIP-I $\uparrow$} \\
& PAC & DP
& PAC & DP
& PAC & DP \\
\hline
3
& \textbf{.037$\pm$.009} & .002$\pm$.002
& \textbf{.301$\pm$.016} & .197$\pm$.006
& \textbf{.522$\pm$.012} & .471$\pm$.010 \\
5
& \textbf{.057$\pm$.010} & .002$\pm$.002
& \textbf{.347$\pm$.018} & .198$\pm$.006
& \textbf{.554$\pm$.013} & .471$\pm$.010 \\
10
& \textbf{.103$\pm$.013} & .002$\pm$.002
& \textbf{.409$\pm$.021} & .200$\pm$.006
& \textbf{.603$\pm$.014} & .471$\pm$.010 \\
20
& \textbf{.165$\pm$.017} & .003$\pm$.002
& \textbf{.458$\pm$.019} & .205$\pm$.007
& \textbf{.650$\pm$.014} & .472$\pm$.010 \\
50
& \textbf{.195$\pm$.021} & .004$\pm$.002
& \textbf{.478$\pm$.018} & .214$\pm$.007
& \textbf{.670$\pm$.014} & .474$\pm$.010 \\
\hline
\end{tabular}
\end{table}

When averaging the results over the 50 CelebA-HQ identities, we observe substantial variation in performance across identities. To reflect this variability rather than reporting only the mean, we present 95\% confidence intervals based on the variation across the 50 identities in Table~\ref{tab:celeba_ti_interval}. The results remain consistent with those in the main paper: PAC achieves higher ISM, DINO-I, and CLIP-I scores across all privacy budgets, whereas DP-SGD performance changes only marginally as $\varepsilon$ increases.

%% file: main.bbl
\begin{thebibliography}{54}
\providecommand{\natexlab}[1]{#1}
\providecommand{\url}[1]{\texttt{#1}}
\expandafter\ifx\csname urlstyle\endcsname\relax
  \providecommand{\doi}[1]{doi: #1}\else
  \providecommand{\doi}{doi: \begingroup \urlstyle{rm}\Url}\fi

\bibitem[Abadi et~al.(2016)Abadi, Chu, Goodfellow, McMahan, Mironov, Talwar, and Zhang]{abadi2016deep}
Martin Abadi, Andy Chu, Ian Goodfellow, H~Brendan McMahan, Ilya Mironov, Kunal Talwar, and Li~Zhang.
\newblock Deep learning with differential privacy.
\newblock In \emph{Proceedings of the 2016 ACM SIGSAC conference on computer and communications security}, pages 308--318, 2016.

\bibitem[Battiston et~al.(2026)Battiston, Yuan, Zhu, and Boncz]{battiston2026simd}
Ilaria Battiston, Dandan Yuan, Xiaochen Zhu, and Peter Boncz.
\newblock Simd-pac-db: Pretty performant pac privacy.
\newblock \emph{arXiv preprint arXiv:2603.15023}, 2026.

\bibitem[Carlini et~al.(2023)Carlini, Hayes, Nasr, Jagielski, Sehwag, Tramer, Balle, Ippolito, and Wallace]{carlini2023extracting}
Nicolas Carlini, Jamie Hayes, Milad Nasr, Matthew Jagielski, Vikash Sehwag, Florian Tramer, Borja Balle, Daphne Ippolito, and Eric Wallace.
\newblock Extracting training data from diffusion models.
\newblock In \emph{32nd USENIX security symposium (USENIX Security 23)}, pages 5253--5270, 2023.

\bibitem[Caron et~al.(2021)Caron, Touvron, Misra, J{\'e}gou, Mairal, Bojanowski, and Joulin]{caron2021emerging}
Mathilde Caron, Hugo Touvron, Ishan Misra, Herv{\'e} J{\'e}gou, Julien Mairal, Piotr Bojanowski, and Armand Joulin.
\newblock Emerging properties in self-supervised vision transformers.
\newblock In \emph{2021 IEEE/CVF international conference on computer vision (ICCV)}, pages 9630--9640. IEEE, 2021.

\bibitem[Deng et~al.(2019)Deng, Guo, Xue, and Zafeiriou]{deng2019arcface}
Jiankang Deng, Jia Guo, Niannan Xue, and Stefanos Zafeiriou.
\newblock Arcface: Additive angular margin loss for deep face recognition.
\newblock In \emph{Proceedings of the IEEE/CVF conference on computer vision and pattern recognition}, pages 4690--4699, 2019.

\bibitem[Dockhorn et~al.(2022)Dockhorn, Cao, Vahdat, and Kreis]{dockhorn2022differentially}
Tim Dockhorn, Tianshi Cao, Arash Vahdat, and Karsten Kreis.
\newblock Differentially private diffusion models.
\newblock \emph{arXiv preprint arXiv:2210.09929}, 2022.

\bibitem[Dong et~al.(2022)Dong, Roth, and Su]{dong2022gaussian}
Jinshuo Dong, Aaron Roth, and Weijie~J Su.
\newblock Gaussian differential privacy.
\newblock \emph{Journal of the Royal Statistical Society Series B: Statistical Methodology}, 84\penalty0 (1):\penalty0 3--37, 2022.

\bibitem[Dwork et~al.(2006)Dwork, Kenthapadi, McSherry, Mironov, and Naor]{dwork2006our}
Cynthia Dwork, Krishnaram Kenthapadi, Frank McSherry, Ilya Mironov, and Moni Naor.
\newblock Our data, ourselves: Privacy via distributed noise generation.
\newblock In \emph{Annual international conference on the theory and applications of cryptographic techniques}, pages 486--503. Springer, 2006.

\bibitem[Ertan et~al.(2026)Ertan, Zhu, Nguyen, van Dijk, and Devadas]{ertan2026paczero}
Murat~Bilgehan Ertan, Xiaochen Zhu, Phuong~Ha Nguyen, Marten van Dijk, and Srinivas Devadas.
\newblock Paczero: Pac-private fine-tuning of language models via sign quantization.
\newblock \emph{arXiv preprint arXiv:2605.06505}, 2026.

\bibitem[Gal et~al.(2022)Gal, Alaluf, Atzmon, Patashnik, Bermano, Chechik, and Cohen-Or]{gal2022image}
Rinon Gal, Yuval Alaluf, Yuval Atzmon, Or~Patashnik, Amit~H Bermano, Gal Chechik, and Daniel Cohen-Or.
\newblock An image is worth one word: Personalizing text-to-image generation using textual inversion.
\newblock \emph{arXiv preprint arXiv:2208.01618}, 2022.

\bibitem[Ghalebikesabi et~al.(2023)Ghalebikesabi, Berrada, Gowal, Ktena, Stanforth, Hayes, De, Smith, Wiles, and Balle]{ghalebikesabi2023differentially}
Sahra Ghalebikesabi, Leonard Berrada, Sven Gowal, Ira Ktena, Robert Stanforth, Jamie Hayes, Soham De, Samuel~L Smith, Olivia Wiles, and Borja Balle.
\newblock Differentially private diffusion models generate useful synthetic images.
\newblock \emph{arXiv preprint arXiv:2302.13861}, 2023.

\bibitem[Gopi et~al.(2021)Gopi, Lee, and Wutschitz]{gopi2021numerical}
Sivakanth Gopi, Yin~Tat Lee, and Lukas Wutschitz.
\newblock Numerical composition of differential privacy.
\newblock \emph{Advances in Neural Information Processing Systems}, 34:\penalty0 11631--11642, 2021.

\bibitem[Helber et~al.(2019)Helber, Bischke, Dengel, and Borth]{helber2019eurosat}
Patrick Helber, Benjamin Bischke, Andreas Dengel, and Damian Borth.
\newblock Eurosat: A novel dataset and deep learning benchmark for land use and land cover classification.
\newblock \emph{IEEE Journal of Selected Topics in Applied Earth Observations and Remote Sensing}, 12\penalty0 (7):\penalty0 2217--2226, 2019.

\bibitem[Ho and Salimans(2022)]{ho2022classifier}
Jonathan Ho and Tim Salimans.
\newblock Classifier-free diffusion guidance.
\newblock \emph{arXiv preprint arXiv:2207.12598}, 2022.

\bibitem[Ho et~al.(2020)Ho, Jain, and Abbeel]{ho2020denoising}
Jonathan Ho, Ajay Jain, and Pieter Abbeel.
\newblock Denoising diffusion probabilistic models.
\newblock \emph{Advances in neural information processing systems}, 33:\penalty0 6840--6851, 2020.

\bibitem[Hu et~al.(2021)Hu, Shen, Wallis, Allen-Zhu, Li, Wang, Wang, and Chen]{hu2021lora}
Edward~J Hu, Yelong Shen, Phillip Wallis, Zeyuan Allen-Zhu, Yuanzhi Li, Shean Wang, Lu~Wang, and Weizhu Chen.
\newblock Lora: Low-rank adaptation of large language models.
\newblock \emph{arXiv preprint arXiv:2106.09685}, 2021.

\bibitem[Hu et~al.(2024)Hu, Wu, Li, Long, Garrido, Ge, Ding, Forsyth, Li, and Song]{hu2024sok}
Yuzheng Hu, Fan Wu, Qinbin Li, Yunhui Long, Gonzalo~Munilla Garrido, Chang Ge, Bolin Ding, David Forsyth, Bo~Li, and Dawn Song.
\newblock Sok: Privacy-preserving data synthesis.
\newblock In \emph{2024 IEEE Symposium on Security and Privacy (SP)}, pages 4696--4713. IEEE, 2024.

\bibitem[Jiang et~al.(2025)Jiang, Li, Ma, and Wang]{jiang2025rapid}
Tanqiu Jiang, Changjiang Li, Fenglong Ma, and Ting Wang.
\newblock Rapid: Retrieval augmented training of differentially private diffusion models.
\newblock \emph{arXiv preprint arXiv:2502.12794}, 2025.

\bibitem[Karras et~al.(2017)Karras, Aila, Laine, and Lehtinen]{karras2017progressive}
Tero Karras, Timo Aila, Samuli Laine, and Jaakko Lehtinen.
\newblock Progressive growing of gans for improved quality, stability, and variation.
\newblock \emph{arXiv preprint arXiv:1710.10196}, 2017.

\bibitem[Kim et~al.(2024)Kim, Song, Castells, and Choi]{kim2024bk}
Bo-Kyeong Kim, Hyoung-Kyu Song, Thibault Castells, and Shinkook Choi.
\newblock Bk-sdm: A lightweight, fast, and cheap version of stable diffusion.
\newblock In \emph{European Conference on Computer Vision}, pages 381--399. Springer, 2024.

\bibitem[Kulynych et~al.(2025)Kulynych, Gomez, Kaissis, Hayes, Balle, Calmon, and Raisaro]{kulynych2025unifying}
Bogdan Kulynych, Juan Gomez, Georgios Kaissis, Jamie Hayes, Borja Balle, Flavio Calmon, and Jean Raisaro.
\newblock Unifying re-identification, attribute inference, and data reconstruction risks in differential privacy.
\newblock \emph{Advances in Neural Information Processing Systems}, 38:\penalty0 111967--112003, 2025.

\bibitem[Kumari et~al.(2023)Kumari, Zhang, Zhang, Shechtman, and Zhu]{kumari2023multi}
Nupur Kumari, Bingliang Zhang, Richard Zhang, Eli Shechtman, and Jun-Yan Zhu.
\newblock Multi-concept customization of text-to-image diffusion.
\newblock In \emph{2023 IEEE/CVF Conference on Computer Vision and Pattern Recognition (CVPR)}, pages 1931--1941. IEEE, 2023.

\bibitem[Li et~al.(2024)Li, Gong, Li, Zhao, Hou, and Wang]{li2024privimage}
Kecen Li, Chen Gong, Zhixiang Li, Yuzhong Zhao, Xinwen Hou, and Tianhao Wang.
\newblock {PrivImage}: Differentially private synthetic image generation using diffusion models with {Semantic-Aware} pretraining.
\newblock In \emph{33rd USENIX Security Symposium (USENIX Security 24)}, pages 4837--4854, 2024.

\bibitem[Li et~al.(2025)Li, Gong, Li, Zhao, Hou, and Wang]{li2025easy}
Kecen Li, Chen Gong, Xiaochen Li, Yuzhong Zhao, Xinwen Hou, and Tianhao Wang.
\newblock From easy to hard: Building a shortcut for differentially private image synthesis.
\newblock In \emph{2025 IEEE Symposium on Security and Privacy (SP)}, pages 3988--4006. IEEE, 2025.

\bibitem[Liu et~al.(2024{\natexlab{a}})Liu, Lowy, Koike-Akino, Parsons, and Wang]{liu2024efficient}
Jing Liu, Andrew Lowy, Toshiaki Koike-Akino, Kieran Parsons, and Ye~Wang.
\newblock Efficient differentially private fine-tuning of diffusion models.
\newblock \emph{arXiv preprint arXiv:2406.05257}, 2024{\natexlab{a}}.

\bibitem[Liu et~al.(2024{\natexlab{b}})Liu, Lyu, Vinaroz, and Park]{liu2024differentially}
Michael~F Liu, Saiyue Lyu, Margarita Vinaroz, and Mijung Park.
\newblock Differentially private latent diffusion models.
\newblock \emph{Transactions on Machine Learning Research}, 2024{\natexlab{b}}.

\bibitem[Liu et~al.(2015)Liu, Luo, Wang, and Tang]{liu2015deep}
Ziwei Liu, Ping Luo, Xiaogang Wang, and Xiaoou Tang.
\newblock Deep learning face attributes in the wild.
\newblock In \emph{Proceedings of the IEEE international conference on computer vision}, pages 3730--3738, 2015.

\bibitem[Matsumoto et~al.(2023)Matsumoto, Miura, and Yanai]{matsumoto2023membership}
Tomoya Matsumoto, Takayuki Miura, and Naoto Yanai.
\newblock Membership inference attacks against diffusion models.
\newblock In \emph{2023 IEEE Security and Privacy Workshops (SPW)}, pages 77--83. IEEE, 2023.

\bibitem[Mirzadehsarcheshmeh and Khandani(2026)]{mirzadehsarcheshmeh2026pac}
Mina Mirzadehsarcheshmeh and Amir~Keyvan Khandani.
\newblock Pac-private autoregressive generation: Calibrating noise to ensemble disagreement.
\newblock \emph{arXiv preprint arXiv:2609.05676}, 2026.

\bibitem[Mittal et~al.(2012)Mittal, Moorthy, and Bovik]{mittal2012no}
Anish Mittal, Anush~Krishna Moorthy, and Alan~Conrad Bovik.
\newblock No-reference image quality assessment in the spatial domain.
\newblock \emph{IEEE Transactions on image processing}, 21\penalty0 (12):\penalty0 4695--4708, 2012.

\bibitem[Nichol et~al.(2021)Nichol, Dhariwal, Ramesh, Shyam, Mishkin, McGrew, Sutskever, and Chen]{nichol2021glide}
Alex Nichol, Prafulla Dhariwal, Aditya Ramesh, Pranav Shyam, Pamela Mishkin, Bob McGrew, Ilya Sutskever, and Mark Chen.
\newblock Glide: Towards photorealistic image generation and editing with text-guided diffusion models.
\newblock \emph{arXiv preprint arXiv:2112.10741}, 2021.

\bibitem[Oquab et~al.(2023)Oquab, Darcet, Moutakanni, Vo, Szafraniec, Khalidov, Fernandez, Haziza, Massa, El-Nouby, et~al.]{oquab2023dinov2}
Maxime Oquab, Timoth{\'e}e Darcet, Th{\'e}o Moutakanni, Huy Vo, Marc Szafraniec, Vasil Khalidov, Pierre Fernandez, Daniel Haziza, Francisco Massa, Alaaeldin El-Nouby, et~al.
\newblock Dinov2: Learning robust visual features without supervision.
\newblock \emph{arXiv preprint arXiv:2304.07193}, 2023.

\bibitem[Podell et~al.(2024)Podell, English, Lacey, Blattmann, Dockhorn, M{\"u}ller, Penna, and Rombach]{podell2024sdxl}
Dustin Podell, Zion English, Kyle Lacey, Andreas Blattmann, Tim Dockhorn, Jonas M{\"u}ller, Joe Penna, and Robin Rombach.
\newblock Sdxl: Improving latent diffusion models for high-resolution image synthesis.
\newblock In \emph{International Conference on Learning Representations}, pages 1862--1874, 2024.

\bibitem[Radford et~al.(2021)Radford, Kim, Hallacy, Ramesh, Goh, Agarwal, Sastry, Askell, Mishkin, Clark, et~al.]{radford2021learning}
Alec Radford, Jong~Wook Kim, Chris Hallacy, Aditya Ramesh, Gabriel Goh, Sandhini Agarwal, Girish Sastry, Amanda Askell, Pamela Mishkin, Jack Clark, et~al.
\newblock Learning transferable visual models from natural language supervision.
\newblock In \emph{International conference on machine learning}, pages 8748--8763. PmLR, 2021.

\bibitem[Repetto et~al.(2026)Repetto, Sheybani, De~Micheli, and Koushanfar]{repetto2026pac}
Guilhem Repetto, Nojan Sheybani, Gabrielle De~Micheli, and Farinaz Koushanfar.
\newblock Pac to the future: Zero-knowledge proofs of pac private systems.
\newblock In \emph{Companion Proceedings of the ACM Web Conference 2026}, pages 1532--1540, 2026.

\bibitem[Rombach et~al.(2022)Rombach, Blattmann, Lorenz, Esser, and Ommer]{rombach2022high}
Robin Rombach, Andreas Blattmann, Dominik Lorenz, Patrick Esser, and Bj{\"o}rn Ommer.
\newblock High-resolution image synthesis with latent diffusion models.
\newblock In \emph{Proceedings of the IEEE/CVF conference on computer vision and pattern recognition}, pages 10684--10695, 2022.

\bibitem[Ruiz et~al.(2023)Ruiz, Li, Jampani, Pritch, Rubinstein, and Aberman]{ruiz2023dreambooth}
Nataniel Ruiz, Yuanzhen Li, Varun Jampani, Yael Pritch, Michael Rubinstein, and Kfir Aberman.
\newblock Dreambooth: Fine tuning text-to-image diffusion models for subject-driven generation.
\newblock In \emph{2023 IEEE/CVF Conference on Computer Vision and Pattern Recognition (CVPR)}, pages 22500--22510. IEEE, 2023.

\bibitem[Shannon(1948)]{shannon1948mathematical}
Claude~Elwood Shannon.
\newblock A mathematical theory of communication.
\newblock \emph{The Bell system technical journal}, 27\penalty0 (3):\penalty0 379--423, 1948.

\bibitem[Song et~al.(2020)Song, Meng, and Ermon]{song2020denoising}
Jiaming Song, Chenlin Meng, and Stefano Ermon.
\newblock Denoising diffusion implicit models.
\newblock \emph{arXiv preprint arXiv:2010.02502}, 2020.

\bibitem[Sridhar et~al.(2025)Sridhar, Xiao, and Devadas]{sridhar2025pac}
Mayuri Sridhar, Hanshen Xiao, and Srinivas Devadas.
\newblock Pac-private algorithms.
\newblock In \emph{2025 IEEE Symposium on Security and Privacy (SP)}, pages 3839--3857. IEEE, 2025.

\bibitem[Sridhar et~al.(2026)Sridhar, Noguera, Jain, Kristensen, Devadas, Xiao, and Yu]{sridhar2026pac}
Mayuri Sridhar, Michael~A Noguera, Chaitanyasuma Jain, Kevin Kristensen, Srinivas Devadas, Hanshen Xiao, and Xiangyao Yu.
\newblock Pac-private databases.
\newblock \emph{Cryptology ePrint Archive}, 2026.

\bibitem[Stadler et~al.(2022)Stadler, Oprisanu, and Troncoso]{stadler2022synthetic}
Theresa Stadler, Bristena Oprisanu, and Carmela Troncoso.
\newblock Synthetic data--anonymisation groundhog day.
\newblock In \emph{31st USENIX Security Symposium (USENIX Security 22)}, pages 1451--1468, 2022.

\bibitem[Tsai et~al.(2025)Tsai, Li, Yu, Ren, Chen, Chen, and Buet-Golfouse]{tsai2025differentially}
Yu-Lin Tsai, Yizhe Li, Chia-Mu Yu, Xuebin Ren, Po-Yu Chen, Zekai Chen, and Francois Buet-Golfouse.
\newblock Differentially private fine-tuning of diffusion models.
\newblock In \emph{Proceedings of the IEEE/CVF International Conference on Computer Vision}, pages 4561--4571, 2025.

\bibitem[Van~Le et~al.(2023)Van~Le, Phung, Nguyen, Dao, Tran, and Tran]{van2023anti}
Thanh Van~Le, Hao Phung, Thuan~Hoang Nguyen, Quan Dao, Ngoc~N Tran, and Anh Tran.
\newblock Anti-dreambooth: Protecting users from personalized text-to-image synthesis.
\newblock In \emph{2023 IEEE/CVF International Conference on Computer Vision (ICCV)}, pages 2116--2127. IEEE, 2023.

\bibitem[Wang et~al.(2024)Wang, Pang, Lu, Rao, Zhou, and Xue]{wang2024dp}
Haichen Wang, Shuchao Pang, Zhigang Lu, Yihang Rao, Yongbin Zhou, and Minhui Xue.
\newblock dp-promise: Differentially private diffusion probabilistic models for image synthesis.
\newblock In \emph{33rd USENIX Security Symposium (USENIX Security 24)}, pages 1063--1080, 2024.

\bibitem[Webster(2023)]{webster2023reproducible}
Ryan Webster.
\newblock A reproducible extraction of training images from diffusion models.
\newblock \emph{arXiv preprint arXiv:2305.08694}, 2023.

\bibitem[Xiao et~al.(2017)Xiao, Rasul, and Vollgraf]{xiao2017fashion}
Han Xiao, Kashif Rasul, and Roland Vollgraf.
\newblock Fashion-mnist: a novel image dataset for benchmarking machine learning algorithms.
\newblock \emph{arXiv preprint arXiv:1708.07747}, 2017.

\bibitem[Xiao and Devadas(2023)]{xiao2023pac}
Hanshen Xiao and Srinivas Devadas.
\newblock Pac privacy: Automatic privacy measurement and control of data processing.
\newblock In \emph{Annual International Cryptology Conference}, pages 611--644. Springer, 2023.

\bibitem[Xiao et~al.(2024)Xiao, Suh, and Devadas]{xiao2024formal}
Hanshen Xiao, G~Edward Suh, and Srinivas Devadas.
\newblock Formal privacy proof of data encoding: The possibility and impossibility of learnable encryption.
\newblock In \emph{Proceedings of the 2024 on ACM SIGSAC Conference on Computer and Communications Security}, pages 1834--1848, 2024.

\bibitem[Xiao et~al.(2025)Xiao, Wan, Shi, and Devadas]{xiao2025one}
Hanshen Xiao, Jun Wan, Elaine Shi, and Srinivas Devadas.
\newblock One-sided bounded noise: Theory, optimization algorithms and applications.
\newblock In \emph{Proceedings of the 2025 ACM SIGSAC Conference on Computer and Communications Security}, pages 4214--4228, 2025.

\bibitem[Xu(2024)]{xu2024pac}
Qipan Xu.
\newblock Pac privacy preserving diffusion models.
\newblock Master's thesis, Rutgers The State University of New Jersey, School of Graduate Studies, 2024.

\bibitem[Yang et~al.(2023)Yang, Shi, Wei, Liu, Zhao, Ke, Pfister, and Ni]{yang2023medmnist}
Jiancheng Yang, Rui Shi, Donglai Wei, Zequan Liu, Lin Zhao, Bilian Ke, Hanspeter Pfister, and Bingbing Ni.
\newblock Medmnist v2-a large-scale lightweight benchmark for 2d and 3d biomedical image classification.
\newblock \emph{Scientific data}, 10\penalty0 (1):\penalty0 41, 2023.

\bibitem[Zhang and Vorobeychik(2026)]{zhang2026residual}
Tao Zhang and Yevgeniy Vorobeychik.
\newblock {Residual-PAC} privacy: Automatic privacy control beyond the gaussian barrier.
\newblock In \emph{35th USENIX Security Symposium (USENIX Security 26)}, 2026.

\bibitem[Zhu(2026)]{zhu2026pac}
Xiaochen Zhu.
\newblock \emph{PAC-private responses with adversarial composition}.
\newblock PhD thesis, Massachusetts Institute of Technology, 2026.

\end{thebibliography}
